%% file: arxiv-main.tex
\documentclass[10pt,twocolumn,letterpaper]{article}

\usepackage{cvpr}              %

\input{preamble}

\definecolor{cvprblue}{rgb}{0.21,0.49,0.74}
\usepackage[pagebackref,breaklinks,colorlinks,citecolor=cvprblue]{hyperref}

\usepackage[accsupp]{axessibility} %

\def\paperID{3802} %
\def\confName{CVPR}
\def\confYear{2024}

\usepackage{times}
\usepackage{epsfig}
\usepackage{graphicx}
\usepackage{amsmath}
\usepackage{amsthm}
\usepackage{amssymb}
\usepackage{bm}
\usepackage[normalem]{ulem}
\usepackage[dvipsnames]{xcolor}
\usepackage{colortbl}
\usepackage{algorithm2e}

\newcommand{\nn}{\nonumber}
\newcommand{\cov}{\mathbb{C}\text{ov}}
\newcommand{\indep}{\perp\!\!\!\perp}
\newtheorem{theorem}{Theorem}

\newtheorem{theorem1}{Theorem}

\def\etal{{\it et al. }}
\theoremstyle{definition}
\newtheorem{definition}{Definition}

\newcommand{\Tr}[1]{\text{Tr}\left\{#1\right\}}

\DeclareMathOperator*{\arginf}{arg\,inf}

\usepackage{times}
\usepackage{epsfig}
\usepackage{graphicx}
\usepackage{amsmath}
\usepackage{amsthm}
\usepackage{amssymb}
\usepackage{bm}
\usepackage[font={small}]{caption}
\usepackage{subcaption}
\usepackage{multirow}
\usepackage{rotating}

\usepackage{changepage}
\usepackage{pgfplotstable} 
\usepgfplotslibrary{fillbetween}
\usepackage{pgfplots}

\usepackage{tikz}

\newcommand{\methodName}[0]{U-FaTE} %

\newcommand\scalemath[2]{\scalebox{#1}{\mbox{\ensuremath{\displaystyle #2}}}}

\usepackage{fp} %

\usepackage{listofitems} %

\usepackage[mode=buildnew]{standalone}%
\usepackage{tikz}

\usepackage{pgfmath}

\usepackage{amsfonts}
\usepackage{amsmath}
\usepackage{times}

\usepackage{pgfplots}
\usepgfplotslibrary{patchplots}
\usepgfplotslibrary{fillbetween}
\usetikzlibrary{plotmarks}
\usepackage[skins]{tcolorbox}

\usepackage{svg}

\pgfplotsset{compat=1.18}

\begin{document}

\title{Utility-Fairness Trade-Offs and How to Find Them}

\author{Sepehr Dehdashtian \quad Bashir Sadeghi \quad Vishnu Naresh Boddeti\\
Michigan State University\\
{\tt\small \{sepehr, sadeghib, vishnu\}@msu.edu}
}

\maketitle

\definecolor{fillcolor}{HTML}{DEDEFF}  %
\definecolor{fillcolor2}{HTML}{CFD3D8} %
\definecolor{fillcolor4}{HTML}{f2f3f5} %
\definecolor{fillcolor5}{HTML}{f7f8fa} %
\definecolor{fillcolor3}{HTML}{FFE342} %
\definecolor{fillcolor6}{HTML}{ff8f8f} %
\definecolor{fillcolor7}{HTML}{54c45e} %
\definecolor{fillcolor8}{HTML}{ff9933} %

\definecolor{colorbox}{HTML}{f21654} %

\input{00-abstract}
\input{01-introduction}

\input{02-related-work}

\input{03-trade-offs}
\input{04-approach}

\input{05-experiments}
\input{07-conclusion}

{\small
\bibliographystyle{ieeenat_fullname}
\bibliography{egbib}
}

\newpage
\onecolumn
\newpage
\input{08-appendix-arxiv}

\end{document}

%% file: preamble.tex
\usepackage[dvipsnames]{xcolor}

%% file: 00-abstract.tex
\begin{abstract}
When building classification systems with demographic fairness considerations, there are two objectives to satisfy: 1) maximizing utility for the specific task and 2) ensuring fairness w.r.t. a known demographic attribute. These objectives often compete, so optimizing both can lead to a trade-off between utility and fairness. While existing works acknowledge the trade-offs and study their limits, two questions remain unanswered: 1) What are the optimal trade-offs between utility and fairness? and 2) How can we numerically quantify these trade-offs from data for a desired prediction task and demographic attribute of interest? This paper addresses these questions. We introduce two utility-fairness trade-offs: the Data-Space and Label-Space Trade-off. The trade-offs reveal three regions within the utility-fairness plane, delineating what is fully and partially possible and impossible. We propose \methodName{}, a method to numerically quantify the trade-offs for a given prediction task and group fairness definition from data samples. Based on the trade-offs, we introduce a new scheme for evaluating representations. An extensive evaluation of fair representation learning methods and representations from over 1000 pre-trained models revealed that most current approaches are far from the estimated and achievable fairness-utility trade-offs across multiple datasets and prediction tasks.
\end{abstract}

%% file: 01-introduction.tex
\section{Introduction\label{sec:intro}}

As learning-based systems are increasingly being deployed in high-stakes applications, there is a dire need to ensure that they do not propagate or amplify any discriminative tendencies inherent in the training datasets. An ideal solution would impart fairness to prediction models while retaining the performance of the same model when learned without fairness considerations.

\begin{figure}[ht]
    \centering
        \subcaptionbox{\label{fig:teaser-ideal}}[0.48\linewidth]{
            \centering
            \includegraphics[width=0.8\linewidth]{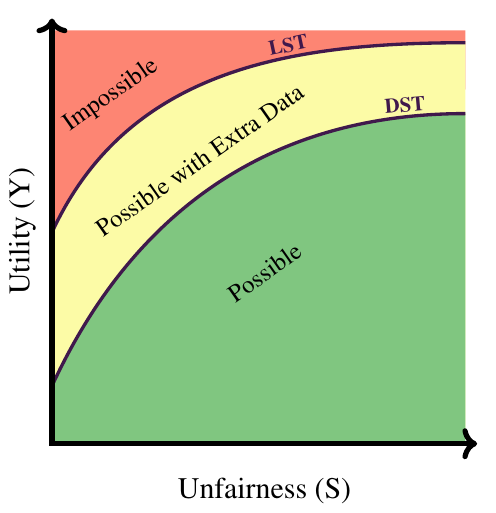}
        }
        \subcaptionbox{\label{fig:teaser-real}}[0.48\linewidth]{
            \centering
            \includegraphics[width=1.0\linewidth]{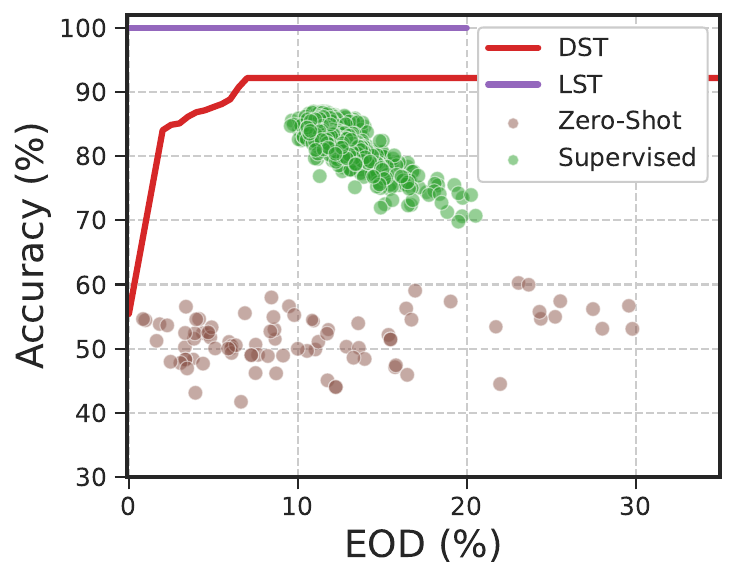}
        }
    \vspace{-0.3cm}
    \caption{\textbf{The utility-fairness trade-offs.} (a) Classification systems can be evaluated by their utility (e.g., accuracy) w.r.t. a target label $Y$ and their unfairness w.r.t. a demographic label $S$. We introduce two trade-offs, \emph{Data Space Trade-Off} (DST) and \emph{Label Space Trade-Off} (LST). (b) We empirically estimate DST and LST on CelebA and evaluate the utility (high cheekbones) and fairness (gender \& age) of over 100 zero-shot and 900 supervised models. \label{fig:trade-offs}}
    \vspace{-1.8em}
\end{figure}

Realizing this goal necessitates optimizing two objectives: maximizing utility in predicting a label $Y$ for a target task (e.g., face identity) while minimizing the unfairness w.r.t. a demographic attribute $S$ (e.g., age or gender). However, when the statistical dependence between $Y$ and $S$ is not negligible, learning with fairness considerations will necessarily degrade the performance of the target predictor, i.e., \emph{a trade-off will exist between utility and fairness}.

The existence of a utility-fairness trade-off has been well established, theoretically~\cite{menon2018cost, zhao2019inherent, zhao2021costs, gouic2020projection, sadeghi2022on} and empirically~\cite{sadeghi2022on}, in multiple prior works. However, the focus of this body of work has been limited in multiple respects. First, prior work~\cite{sadeghi2022on, zhao2021costs} focused on just one type of trade-off, ignoring other possible trade-offs between utility and fairness. Second, prior work~\cite{zhao2019inherent, zhao2021costs} focused on establishing bounds or identifying the end-points of the trade-off of interest rather than attempting its precise characterization. Third, the majority of the prior work~\cite{gouic2020projection, zhao2019inherent, zhao2021costs, sadeghi2022on} has investigated the utility-fairness trade-offs for one definition of group fairness, namely, demographic parity (DP). There are multiple fairness definitions~\cite{barocas2019fairness}, including those more practically relevant than DP, such as Equalized Opportunity (EO), for which the trade-offs have not been studied.

Despite these attempts, several questions related to the utility-fairness trade-offs remain outstanding.
\begin{tcolorbox}[left=0pt,right=0pt,top=0pt,bottom=0pt,colback=colorbox!5!white,colframe=colorbox!75!black]
    \begin{enumerate}
        \item \emph{What are the optimal utility-fairness trade-offs?}
        \item \emph{For a given prediction task and a demographic attribute, we wish to be fair w.r.t., how can we empirically estimate the trade-offs from data?}
    \end{enumerate}
\end{tcolorbox}

Addressing these questions by identifying and quantifying the trade-offs is the primary goal of this paper. The trade-offs are a function of the data triplet $(X, Y, S)$, where $X$ is the data (e.g., images), $Y$ is the target label, and $S$ is the sensitive demographic label. Figure~\ref{fig:trade-offs} illustrates the plausible trade-offs, their empirical estimation on CelebA~\cite{liu2015deep}, and their utility in empirically evaluating representations from pre-trained models.

\vspace{3pt}
\noindent\textbf{Identifying Trade-Offs (\S\ref{sec:trade-offs}).} We identify two trade-offs: the \emph{Label-Space Trade-Off} (LST) and \emph{Data-Space Trade-Off} (DST). They can be defined for \emph{any} group fairness definitions that can be expressed via \emph{independence} and \emph{separation} relations~\cite{barocas2019fairness}. The LST corresponds to the trade-off obtained by an \emph{oracle fair classifier} that depends only on the distributions of $Y$ and $S$. Similarly, DST is the trade-off obtained by an \emph{optimally learned fair classifier} and depends on $(X, Y, S)$. By definition, LST necessarily dominates DST since it does not depend on the data $X$.

The trade-offs divide the utility-fairness plane into three regions shown in Fig.~\ref{fig:teaser-ideal}. A \emph{possible} region that can be attained by algorithms learned on the given data, a \emph{possible with extra data} region that can be attained by learning on data beyond the given data, and an \emph{impossible} region that cannot be attained by any algorithmic scheme due to the inherent dependence between the distributions of $Y$ and $S$.

\vspace{3pt}
\noindent\textbf{Quantifying Trade-Offs (\S\ref{sec:approach}).} Characterizing the exact trade-offs from data for a given task, demographic attribute, and fairness definition affords multiple benefits. It will allow researchers and practitioners to identify the achievable solution space for the task, evaluate how far a given predictor is from the optimal solution, and identify performance gaps and trends among existing solutions. To this end, we propose \methodName{} (\underline{U}tility-\underline{Fa}irness \underline{T}rade-Off \underline{E}stimator), a method for quantifying the trade-offs from data triplets numerically. \methodName{} is an end-to-end model that adopts a statistical dependence measure as a proxy for utility and fairness and optimizes their weighted linear combination. \methodName{} can be flexibly adapted to estimate both the DST and LST from a finite labeled dataset.

\vspace{3pt}
\noindent\textbf{Usefulness of Trade-Offs (\S\ref{sec:metrics}).} The trade-offs illuminate the fundamental limits of learning algorithms in mitigating unfairness and present a new avenue to evaluate a given image representation in terms of its distance from the estimated trade-offs. We adopt this scheme to evaluate the representations of over 900 supervised and 100 zero-shot publicly available pre-trained models, derive insights, and identify trends and models that are close and far from the empirical trade-off estimates (\S\ref{sec:experiments}).

\vspace{3pt}
\noindent\textbf{Notation:} We denote scalars using lowercase letters, e.g., $d$ and $\lambda$. We denote deterministic vectors by boldface lowercase letters, e.g., $\bm x$, $\bm y$. Both scalar-valued and multidimensional random variables (RV)s are denoted by regular upper case letters, e.g., $X$, $Y$. We denote deterministic matrices by boldface upper case letters, e.g., $\bm K$, $\bm \Theta$. Finite or infinite sets are denoted by calligraphic letters, e.g., $\mathcal A$, $\mathcal H$.

%% file: 02-related-work.tex
\section{Related Works\label{sec:related-work}}
A vast majority of prior work on designing fair classifiers focused primarily on uncovering disparities in practical tasks~\cite{buolamwini2018gender, wang2019balanced} and learning a fair predictor~\cite{zemel2013learning, madras2018learning, wang2020towards} for a given fairness measure. An extended discussion of this body of work can be found in the supplementary material.

\vspace{3pt}
\noindent\textbf{Utility-Fairness Trade-Offs:} Many attempts on learning fair models~\cite{madras2018learning, wang2019balanced, wang2020towards, gong2021mitigating} ignored the existence of trade-offs. They sought to maximize accuracy on target tasks while minimizing unfairness, thus perhaps seeking an infeasible solution. Most studies on utility-fairness trade-offs are theoretical and under restricted settings in terms of the type of labels, notion of fairness, and bounds or extreme limits of trade-offs. For example, Zhao~\etal~\cite{zhao2019trade} obtained a lower bound on DST when both $Y$ and $S$ are binary labels. McNamara~\etal~\cite{mcnamara2019costs} provided both upper and lower bounds for binary labels. Only a couple of attempts~\cite{sadeghi2019global, sadeghi2022on} have been made to numerically estimate utility-fairness trade-offs for \emph{independence} related-based measures like demographic parity, both of them on features from pre-trained models, rather than raw data. Sadeghi~\etal~\cite{sadeghi2019global} obtained a simplified version of DST, but for linear models. Later on, in the context of invariant representation learning, this was extended to estimate a near-optimal DST-like trade-off called $\mathcal T_{\text{Opt}}$ in~\cite{sadeghi2022on}. But as we demonstrate in \S\ref{sec:experiments:results-frl}, the estimate of $\mathcal T_{\text{Opt}}$ called K-$\mathcal T_{\text{Opt}}$ does not span the entire trade-off.

In contrast to this body of work, we identify two types of trade-offs, DST and LST, and propose a method to numerically quantify them from data. Our trade-offs and their empirical estimates apply to a wide range of prediction tasks for two different categories of fairness notions without any restrictions on the type of labels.

\vspace{3pt}
\noindent\textbf{Learning Fair Classifiers:} Over the last decade, many methods have been developed for learning fair classifiers. These approaches follow the template of adopting a fairness constraint as a regularizer in addition to the objective for the target task. The approaches differ in the choice of measure as a proxy for quantifying the level of unfairness between the target label $Y$ and the prediction $\hat{Y}$, and the associated optimization technique. From an optimization perspective, they can be classified into three major categories--i.e., iterative adversarial methods (ARL\cite{xie2017controllable}, SARL\cite{sadeghi2019global}, and MaxEnt-ARL\cite{roy2019mitigating}), non-iterative adversarial methods (FairHSIC~\cite{quadrianto2019discovering}, OptNet-ARL~\cite{sadeghi2021adversarial}), and closed-form solver methods (SARL~\cite{sadeghi2019global}, K-$\mathcal T_{\text{Opt}}$~\cite{sadeghi2022on}, LEACE~\cite{belrose2023leace}, FairerCLIP~\cite{dehdashtian2024fairerclip}). Among these, ARL, SARL, MaxEnt-ARL, and OptNet-ARL measure mean dependence~\cite{grari2020learning, adeli2021representation}, FairHSIC, FairerCLIP and K-$\mathcal T_{\text{Opt}}$ measure full statistical dependence, i.e., all modes of dependence, and SARL and LEACE measures linear dependence. %

\methodName{} draws inspiration from K-$\mathcal T_{\text{Opt}}$~\cite{sadeghi2022on}. By using a closed-form solver and a universal dependence measure that captures all non-linear dependencies, K-$\mathcal T_{\text{Opt}}$ achieves a better utility-fairness trade-off and is more stable than the other fair learning methods discussed above. However, K-$\mathcal T_{\text{Opt}}$ is limited in multiple respects and cannot be directly employed for estimating the trade-offs. 1) It operates on features and does not generalize to learning directly from high-dimensional raw data representations such as pixels for images. 2) K-$\mathcal T_{\text{Opt}}$ optimizes an unconditional dependence measure, which limits its applicability to \emph{independence} relation-based fairness definitions such as demographic parity. 3) As we demonstrate in \S\ref{sec:experiments:results-frl}, for demographic parity, K-$\mathcal T_{\text{Opt}}$'s trade-off is the closest to our DST estimate, but it does not span the entire utility-fairness trade-off front. Therefore, we adopt the positive aspects of K-$\mathcal T_{\text{Opt}}$, namely universal dependence measure and closed-form solver, into \methodName{} and overcome its drawbacks.

%% file: 03-trade-offs.tex
\section{The Utility-Fairness Trade-Offs \label{sec:trade-offs}}

\noindent\textbf{Fairness Notions:\label{sec:background:diff-def-fairness}} Group fairness notions are typically categorized into three classes~\cite{barocas2019fairness}, namely \emph{independence}, \emph{separation}, and \emph{sufficiency}, each corresponding to different societal desiderata. We focus on the \emph{idependence} and \emph{separation} relations, which can be expressed as independence ($\hat{Y}\indep S$) and conditional independence ($\hat{Y}\indep S|Y=y$) relations, respectively.

We consider frequently used fairness criteria including Demographic Parity (DP)~\cite{kilbertus2017avoiding} which is an example of an \emph{independence} relation, and Equalized Opportunity (EO)~\cite{hardt2016equality}, and Equality of Odds (EOO)~\cite{hardt2016equality} which are both examples of \emph{separation} relations. The corresponding unfairness metrics are Demographic Parity Violation $DPV:= |P(\hat{Y}=1|S=0) - P(\hat{Y}=1|S=1)|$, Equalized Opportunity Difference $EOD:= |P(\hat{Y}=1|Y=1, S=0) - P(\hat{Y}=1|Y=1, S=1)|$, and Equality of Odds Difference $EOOD := \frac{1}{2} \sum_{y \in \{0, 1\}}|P(\hat{Y}=1|Y=y, S=0) - P(\hat{Y}=1|Y=y, S=1)|$, respectively.

We now introduce the \emph{Data-Space Trade-Off} and the \emph{Label-Space Trade-Off}. In both, we employ a dependence measure $\mathrm{Dep}(\cdot, \cdot|\cdot)$ to enforce the \emph{independence} and \emph{separation} based fairness constraints. The function $\mathrm{Dep}(\cdot, \cdot|\cdot)\geq 0$ is a parametric or non-parametric measure of statistical dependence. $\mathrm{Dep}(P, Q|R=r) = 0$ implies that conditioned on $R=r$, the random variables (RVs) $P$ and $Q$ are independent. $\mathrm{Dep}(P, Q|R=r) > 0$ means that conditioned on $R=r$, $P$, and $Q$ are dependent, with larger values indicating larger degrees of dependence. When $r$ is the empty set $\emptyset$, we assume that $\mathrm{Dep}(P, Q|R=\emptyset)$ simply reduces to the unconditional dependence $\mathrm{Dep}(P, Q)$.

\begin{tcolorbox}[left=0pt,right=0pt,top=0pt,bottom=0pt,colback=colorbox!5!white,colframe=colorbox!75!black,title=\begin{definition}\label{def:dst}\emph{Data Space Trade-Off} (DST)\end{definition}]
\footnotesize{
\begin{eqnarray}\label{eq:data}
f^{DST}_{\lambda} := \arginf_{f\in \mathcal H_X} \Big\{(1-\lambda)\inf_{g_Y\in \mathcal H_Y }\mathbb E_{X,Y}\left[ \mathcal L_Y\left (g_Y\left(\bm f(X)\right), Y \right)\right] \nn\\
+ \lambda\, \mathrm{Dep}\big(\bm f(X), S | Y=y\big) \Big\}, \quad 0\le\lambda<1\nn
\end{eqnarray}}
\end{tcolorbox}
\noindent Here $f$ is the encoder that maps data $X$ to a representation $Z$, and $g$ is a classifier that predicts $\hat{Y}$ from $Z$. $\mathcal{H}_X$ and $\mathcal{H}_Y$ are the hypothesis classes of functions for $f$ and $g$ respectively. $\mathcal{L}_Y(\cdot, \cdot)$ is the loss function corresponding to the utility, and $\lambda$ controls the trade-off between utility and fairness, i.e., $\lambda=0$ corresponds to ignoring the fairness constraint and only optimizing the utility, while, $\lambda\rightarrow 1$ corresponds to the total fairness. The outcome $f^{DST}_{\lambda}$ corresponds to the encoder for a given value of $\lambda$. This definition corresponds to the \textbf{DST} curve in \cref{fig:teaser-ideal}, where the utility-fairness plane below the DST corresponds to the region achievable by algorithms designed for this prediction task that learn from the data triplet $(X, Y, S) \sim p(X, Y, S)$.

\begin{tcolorbox}[left=0pt,right=0pt,top=0pt,bottom=0pt,colback=colorbox!5!white,colframe=colorbox!75!black,title=\begin{definition}\label{def:lst}\emph{Label Space Trade-Off} (LST)\end{definition}]
\footnotesize{
\begin{eqnarray}\label{eq:label}
Z^{LST}_{\lambda} := \arginf_{Z\in L^2}\Big\{(1-\lambda) \inf_{g_Y \in \mathcal H_Y}\mathbb E_{Y}\Big[ \mathcal L_Y\big (g_Y(Z), Y \big)\Big] \nn\\ 
+\lambda\, \mathrm{Dep}\big(Z, S  | Y=y \big) \Big\}, \quad 0\le\lambda<1\nn
\end{eqnarray}}
\end{tcolorbox}
\noindent Here $L^2$ is the space of all square-integrable RVs \big(i.e. $\mathbb E_{Z}\left[\bm \|Z\|^2\right]<\infty$\big) in the probability space generated by the joint RV $(Y, S)$. LST corresponds to the trade-off from an \emph{ideal} representation space $Z^{LST}_{\lambda}$ that is not constrained to be learned from the input data $X$. It is the trade-off \emph{inherent} to the task itself and is the best that \emph{any} algorithm can hope to achieve for this task. Therefore, it necessarily dominates (or is equal to) DST in \cref{def:dst}. This definition corresponds to the LST curve in \cref{fig:teaser-ideal}, where the utility-fairness plane above the LST corresponds to the region that \emph{any} algorithm cannot achieve.

We stress that the above trade-offs are intrinsic to the underlying data, specifically the underlying distributions that generated that data. So, \emph{the trade-offs are a property of the data, not of any particular learning algorithm}.
\input{figs/ov-inference}
\vspace{3pt}
\noindent\textbf{The LST and DST Divide:\label{sec:gap}} As illustrated by the yellow region in \cref{fig:teaser-ideal}, there is a potential gap between LST and DST. This gap at $\lambda=0$ stems from the irreducible error from the prediction $\mathbb{E}[Y|X]$ of a Bayes Classifier, or when $Y$ is fully recoverable from $X$. And, when $\lambda > 0$, the gap between LST and DST widens in two scenarios: 1) $Y\not\indep S$: The model starts discarding $S$ from the representation $Z$, which will lead to $Y$ being even less recoverable from $X$ compared to $\lambda=0$. and 2) $Y\indep S$: If $X$ entangles $Y$ and $S$ in such a way that $Y$ is not recoverable from $X$ when $S$ is discarded, it will lead to $Y$ being even less recoverable from $X$ compared to $\lambda=0$.

%% file: figs/ov-inference.tex
\FPset\figlinewidth{1.0}
\FPset\figlinewidthsmaller{0.7}

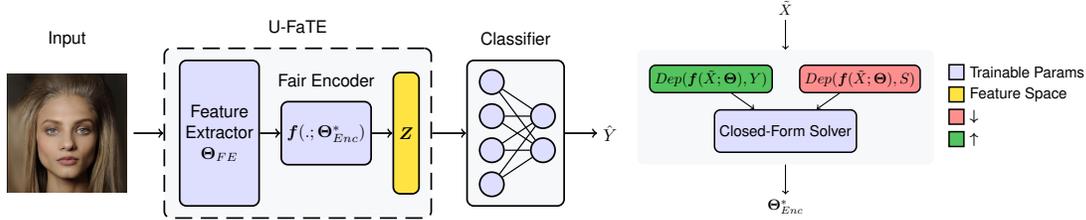
\begin{figure*}[!ht]
  \centering
  \resizebox{0.48\linewidth}{!}{
  \begin{tikzpicture}[node distance=3cm, every node/.style={font=\sffamily}]
    \node[anchor=east, yshift=-3em](input2){\includegraphics[width=2.5cm]{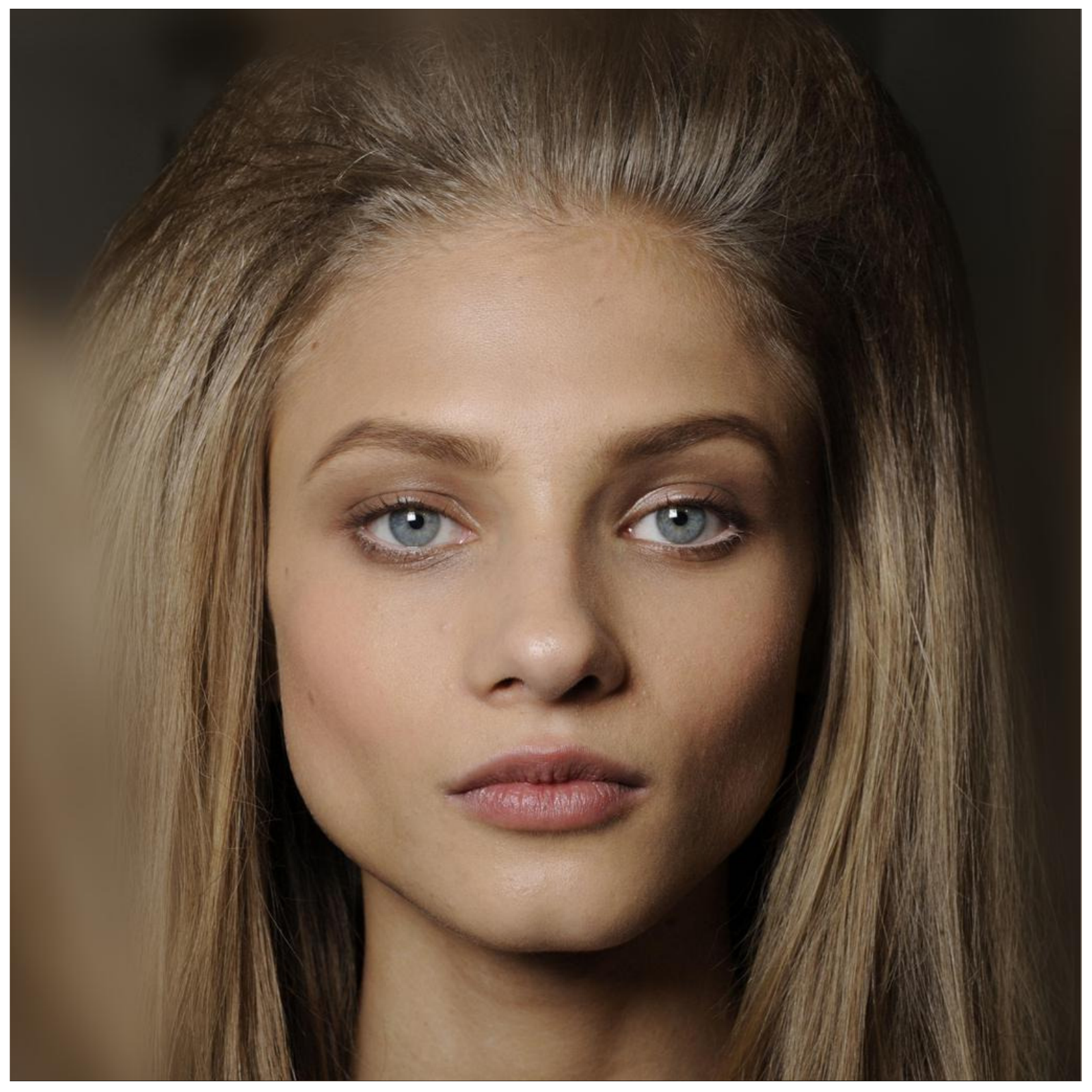}};   
    \node[above of=input2, yshift=-3em]{Input};
    
    \node[draw=black, dash pattern=on 10pt off 5pt, fill=fillcolor5, line width=\figlinewidth pt, rectangle, rounded corners, align=center, anchor=center, xshift=5em, right of=input2, minimum height=3.5cm, minimum width=5.5cm] (box2) {};
    
    \node[draw=black, fill=fillcolor, line width=\figlinewidth pt, rectangle, rounded corners, align=center, anchor=west, minimum height=3cm, minimum width=1cm] (fe2) at ([xshift=0.9em]box2.west) {Feature\\Extractor\\$\bm \Theta_{FE}$};
    
    \node[anchor=south, align=center] () at ([yshift=0.5em, xshift=0em]box2.north){\methodName};

    \node[draw=black, fill=fillcolor, line width=\figlinewidth pt, rectangle, rounded corners, align=center, anchor=west, minimum height=1.3cm, minimum width=0.5cm] (enc) at ([yshift=0em, xshift=1.2em]fe2.east) {$\bm f(.;\bm \Theta^*_{Enc})$}; %
    \node[anchor=south, align=center] () at ([yshift=0.5em, xshift=0em]enc.north){Fair Encoder};

    \node[draw=black, fill=fillcolor3, line width=\figlinewidth pt, rectangle, rounded corners, align=center, anchor=west, minimum height=2.5cm, minimum width=0.5cm] (z) at ([yshift=0em, xshift=1.2em]enc.east) {$\bm Z$};

    \node[draw=black, fill=fillcolor5, line width=\figlinewidth pt, rectangle, rounded corners, align=center, anchor=west, minimum height=3.0cm, minimum width=2.0cm] (box) at ([xshift=2em, yshift=0]box2.east){};
    
    \node[above of=box, yshift=-3em]{Classifier};

    \node[draw=black, fill=fillcolor, line width=\figlinewidth pt, circle, rounded corners, align=center, anchor=center, minimum height=0.5cm] (n11) at ([yshift=3em, xshift=1.5em]box.west) {};
    \node[draw=black, fill=fillcolor, line width=\figlinewidth pt, circle, rounded corners, align=center, anchor=center, minimum height=0.5cm] (n12) at ([yshift=-2em, xshift=0em]n11) {};
    \node[draw=black, fill=fillcolor, line width=\figlinewidth pt, circle, rounded corners, align=center, anchor=center, minimum height=0.5cm] (n13) at ([yshift=-2em, xshift=0em]n12) {};
    \node[draw=black, fill=fillcolor, line width=\figlinewidth pt, circle, rounded corners, align=center, anchor=center, minimum height=0.5cm] (n14) at ([yshift=-2em, xshift=0em]n13) {};
    
    \node[draw=black, fill=fillcolor, line width=\figlinewidth pt, circle, rounded corners, align=center, anchor=center, minimum height=0.5cm] (n21) at ([yshift=0em, xshift=3em]n12) {};
    \node[draw=black, fill=fillcolor, line width=\figlinewidth pt, circle, rounded corners, align=center, anchor=center, minimum height=0.5cm] (n22) at ([yshift=0em, xshift=3em]n13) {};

    \node[right of=box, xshift=-3.0em] (Y) {$\hat{Y}$};

    \draw[line width=\figlinewidth pt, transform canvas={yshift=0mm},draw] [->] (input2) -- (box2);
    \draw[line width=\figlinewidth pt, transform canvas={yshift=0mm},draw] [->] (box2) -- (box);
    \draw[line width=\figlinewidth pt, transform canvas={yshift=0mm},draw] [->] (box) -- (Y);

    \draw[line width=\figlinewidth pt, transform canvas={yshift=0mm},draw] [->] (fe2) -- (enc);
    \draw[line width=\figlinewidth pt, transform canvas={yshift=0mm},draw] [->] (enc) -- (z);

    \draw[line width=\figlinewidthsmaller pt, transform canvas={yshift=0mm},draw] [-] (n11) -- (n21);
    \draw[line width=\figlinewidthsmaller pt, transform canvas={yshift=0mm},draw] [-] (n11) -- (n22);
    
    \draw[line width=\figlinewidthsmaller pt, transform canvas={yshift=0mm},draw] [-] (n12) -- (n21);
    \draw[line width=\figlinewidthsmaller pt, transform canvas={yshift=0mm},draw] [-] (n12) -- (n22);
    
    \draw[line width=\figlinewidthsmaller pt, transform canvas={yshift=0mm},draw] [-] (n13) -- (n21);
    \draw[line width=\figlinewidthsmaller pt, transform canvas={yshift=0mm},draw] [-] (n13) -- (n22);
    
    \draw[line width=\figlinewidthsmaller pt, transform canvas={yshift=0mm},draw] [-] (n14) -- (n21);
    \draw[line width=\figlinewidthsmaller pt, transform canvas={yshift=0mm},draw] [-] (n14) -- (n22);

  \end{tikzpicture}
  }
\resizebox{0.35\linewidth}{!}{
\begin{tikzpicture}[node distance=3cm, every node/.style={font=\sffamily}]
    \node[anchor=south, align=center](input) {$\Tilde{X}$};  
    \node[anchor=north, align=center, rectangle, rounded corners, fill=fillcolor5, minimum height=2.7cm, minimum width=7cm](box) at ([yshift=-2em]input.south) {};  

    \node[draw=black, fill=fillcolor7, line width=\figlinewidth pt, rectangle, rounded corners, align=center, anchor=north, minimum height=0.5cm, minimum width=0.5cm] (depy) at ([yshift=-1em, xshift=-5em]box.north) {$Dep(\bm f(\Tilde{X};\bm \Theta), Y)$};
    \node[draw=black, fill=fillcolor6, line width=\figlinewidth pt, rectangle, rounded corners, align=center, anchor=north, minimum height=0.5cm, minimum width=0.5cm] (deps) at ([yshift=-1em, xshift=5em]box.north) {$Dep(\bm f(\Tilde{X};\bm \Theta), S)$};

    \node[draw=black, fill=fillcolor, line width=\figlinewidth pt, rectangle, rounded corners, align=center, anchor=north, minimum height=1cm, minimum width=0.5cm] (cfs) at ([yshift=-4em, xshift=0em]box.north) {Closed-Form Solver};

    \node[align=center, anchor=north] (out) at ([yshift=-2em]box.south) {$\bm \Theta^*_{Enc}$};
    
    \draw[line width=\figlinewidthsmaller pt, transform canvas={yshift=0mm},draw] [->] (input) -- (box);
    \draw[line width=\figlinewidthsmaller pt, transform canvas={yshift=0mm},draw] [->] (deps) -- (cfs);
    \draw[line width=\figlinewidthsmaller pt, transform canvas={yshift=0mm},draw] [->] (depy) -- (cfs);
    \draw[line width=\figlinewidthsmaller pt, transform canvas={yshift=0mm},draw] [->] (box) -- (out);

     \def\w {0.35}
     \node[draw=black, fill=fillcolor, line width=\figlinewidth pt, rectangle, align=center, anchor=south, minimum height=\w cm, minimum width=\w cm] (leg2) at ([yshift=1.8em, xshift=1.5em]box.east) {};
     \node[right] (txt2) at ([yshift=0em, xshift=0em]leg2.east) {Trainable Params};
     
     \node[draw=black, fill=fillcolor3, line width=\figlinewidth pt, rectangle, align=center, anchor=south, minimum height=\w cm, minimum width=\w cm] (leg3) at ([yshift=-1.5em, xshift=0em]leg2.south) {};
     \node[right] (txt3) at ([yshift=0em, xshift=0em]leg3.east) {Feature Space};
     
     \node[draw=black, fill=fillcolor6, line width=\figlinewidth pt, rectangle, align=center, anchor=south, minimum height=\w cm, minimum width=\w cm] (leg5) at ([yshift=-1.5em, xshift=0.0em]leg3.south) {};
     \node[right] (txt5) at ([yshift=0em, xshift=0em]leg5.east) {$\bm \downarrow$};
     
     \node[draw=black, fill=fillcolor7, line width=\figlinewidth pt, rectangle, align=center, anchor=south, minimum height=\w cm, minimum width=\w cm] (leg6) at ([yshift=-1.5em, xshift=0.0em]leg5.south) {};
     \node[right] (txt6) at ([yshift=0em, xshift=0em]leg6.east) {$\bm \uparrow$};
     
  \end{tikzpicture}
  }
  \caption{\textbf{Overview of \methodName{}:} (Left) It comprises two components, a feature extractor and a fair encoder, that are trained end-to-end. Once \methodName{} is trained, the MLP classifier is trained to predict $Y$ from which fairness metrics can be computed. (Right) The fair encoder parameters are optimized through a closed-form solver operating on the features from the feature extractor. See text for more details.}
  \label{fig:model-highlevel}
\end{figure*}

%% file: 04-approach.tex
\section{Numerically Quantifying the Trade-Offs \label{sec:approach}}
Now, we turn to the second goal of this paper, numerically quantifying the trade-offs from data. \cref{fig:model-highlevel} shows a high-level overview of \methodName{} to learn a fair representation for a given trade-off parameter $\lambda$. \methodName{} comprises a feature extractor and a fair encoder. It receives raw data as input and uses a feature extractor to provide features for the fair encoder. The encoder uses the extracted features and employs a closed-form solver to find the optimum function that maps these features to a new feature space that minimizes the dependency on the sensitive attribute while maximizing the dependency on the target attribute. Following this, to predict the target $Y$, a classifier is trained with the standard cross-entropy loss for classification problems. This process is repeated for multiple values of $\lambda$ with $0 \leq \lambda < 1$ to obtain the full trade-off curves.

\subsection{Problem Setup}
We start from \cref{def:dst} and model the function $\bm{f}$ as a composition $\bm{f}_{FE} \circ \bm{f}_{Enc}$ of the feature extractor and a fair encoder i.e., $\bm{f}(X;\bm{\Theta})=f_{Enc}(f_{FE}(X;\bm{\Theta}_{FE});\bm{\Theta}_{Enc})$. We parameterize $\bm{f}$ with $\bm{\Theta} = [\bm{\Theta}_{FE}; \bm{\Theta}_{Enc}]$ where $\bm{\Theta}_{FE}$ are the parameters of $\bm{f}_{FE}$ and $\bm{\Theta}_{Enc}$ are the parameters of $\bm{f}_{Enc}$. The objective function in \cref{def:dst} is now
\begin{eqnarray}\label{eq:main-formulation}
\min_{\bm \Theta} \Big\{(1-\lambda)\inf_{\bm \Theta_Y }\mathbb E_{X,Y}\left[  L_Y\left (g_Y\left(\bm f(X; \bm \Theta); \bm \Theta_Y\right), Y \right)\right]\nn\\
+ \lambda\, \mathrm{Dep}\left(\bm f\left(X_c; \bm \Theta\right), S_c\right) \Big\}, \quad 0\le\lambda<1.
\end{eqnarray}
where $\mathrm{Dep}\left(\bm f\left(X_c; \bm \Theta\right), S_c\right)$ is equivalent to the term $\mathrm{Dep}(\bm{f}(X), S|Y=y)$ in \cref{def:dst} when $Y$ is not a continuous label. In this case, $X_c\sim P(X|Y=y)$ and $S_c\sim P(S|Y=y)$ are the random variables that represent the data and sensitive attribute conditioned on $Y = y$, respectively. The fair representation is $Z=\bm{f}(X;\bm{\Theta})$.

\subsection{Optimization via Dependence Measures}
The formulation in \eqref{eq:main-formulation} can be directly optimized for an appropriate choice of dependence measure. Different choices of $\mathrm{Dep}$ lead to different fair representation learning methods. For instance, measuring $\mathrm{Dep}$ through an adversary leads to the class of adversarial representation learning (ARL) methods~\cite{xie2017controllable, roy2019mitigating, sadeghi2019global, sadeghi2021adversarial}. Similarly, employing the Hilbert Schmidt Independence Criterion (HSIC)~\cite{gretton2005kernel} as $\mathrm{Dep}$ leads to FairHSIC~\cite{quadrianto2019discovering}. However, due to challenges in optimization~\cite{roy2019mitigating, sadeghi2019global, sadeghi2021adversarial} and as we demonstrate in \S\ref{sec:experiments:results-frl}, these approaches are either very unstable, fail to span the trade-off or lead to sub-optimal trade-offs.

Recently, Sadeghi~\etal~\cite{sadeghi2022on} demonstrated that adopting an HSIC-like dependence measure for the fairness objective and the target loss leads to a closed-form solution that is both efficient and effective at finding a near-optimal trade-off. Therefore, we incorporate the HSIC-like dependence measure and the closed-form solver into \methodName{}. Thus \eqref{eq:main-formulation} can be expressed as,
\begin{eqnarray}
\sup_{\bm f \in \mathcal A_r} \Big\{&(1-\lambda)\,\mathrm{Dep}\left(\bm f({X};\bm{\Theta}), Y\right) \nn \\
&-\lambda\, \mathrm{Dep}\left(\bm f({X_c};\bm{\Theta}), S_c\right)\Big\},
\label{eq:main-kernel}
\end{eqnarray}
\noindent where $\mathcal A_r$ is a function space that encourages the representations to be uncorrelated. It does not affect the optimality of the learned encoder~\cite{sadeghi2022on} and improves the compactness of representation~\cite{bengio2013representation}. Note that while the first term involves all data $X$, the second involves the conditional data $X_c$.

\vspace{3pt}
\noindent\textbf{Choice of Dependence Measure:} We adapt the dependence measure from \cite{sadeghi2022on} since it lends itself to a closed-form solution while capturing linear and non-linear dependencies under mild assumptions. While the dependence measure in \cite{sadeghi2022on} has been defined for absolute independence, our formulation in \eqref{eq:main-kernel} also requires conditional independence to be compatible with \emph{separation} based fairness definitions. Therefore, when $Y$ is not a continuous label, we define the conditional dependence measure as,
\begin{equation}
\begin{aligned}\label{eq:dep-pop}
& \text{Dep}(\bm f(X), S|Y=y) := \\
& \sum_{j=1}^r \sum_{\beta_S \in \mathcal U_S } \mathbb{E}\left[\left(f_j(X_c)-\mathbb{E}f_j(X_c)\right)\left(\beta_S(S_c)-\mathbb{E}\beta_S(S_c)\right)\right]
\end{aligned}
\end{equation}
\noindent where $\mathcal{U}_S$ is a countable orthonormal basis set for the separable universal RKHS $\mathcal{H}_S$ and $X_c\sim P(X|Y=y)$ and $S_c\sim P(S|Y=y)$ are data and sensitive attributes, respectively. \emph{Empirically} it can be estimated as,
\begin{equation}\label{eq:dep-emp}
\text{Dep}(\bm f(X), S|Y=y):=\frac{1}{n^2}\left\|\bm \Theta \bm K_{X_c} \bm H \bm L_{S_c} \right\|^2_F,
\end{equation}
where $n$ is the number of data samples, $\bm K_{X_c} \in \mathbb{R}^{n\times n}$ is the Gram matrix corresponding to $\mathcal{H}_X$, $\bm \Theta$ is the encoder parameter in $\bm f(X) = \bm \Theta [ k_{X_1}, k_{X_2}, \cdots, k_{X_n}]^T$, $\bm H = \bm I_n-\frac{1}{n} \bm 1_n \bm 1_n^T$ is the centering matrix, and $\bm L_{S_c}$ is a full column-rank matrix such that $\bm L_{S_c} \bm L_{S_c}^T=\bm K_{S_c}$ (Cholesky factorization).

\subsection{A Solution to the Optimization Problem}
\noindent\textbf{Closed-Form Solver via Functions in RKHSs:} Directly solving for all the parameters $\bm{\Theta}$ through \eqref{eq:main-kernel} and \eqref{eq:dep-emp} leads to abysmal performance in practice since the kernel $\bm{K}_X$ has to be computed over the raw data space. Therefore, we instead define the fair encoder on the co-domain of the feature extractor $\bm{f}(\cdot;\bm{\Theta}_{FE})$. So, in this case, \eqref{eq:main-kernel} reduces to,
\begin{eqnarray}
\sup_{\bm f_{Enc} \in \mathcal A_r} \Big\{&(1-\lambda)\,\mathrm{Dep}\left(\bm f_{Enc}(\tilde{X};\bm{\Theta}_{Enc}), Y\right) \nn\\
 &-\lambda\, \mathrm{Dep}\left(\bm f_{Enc}(\tilde{X}_c;\bm{\Theta}_{Enc}), S_c\right)\Big\},
\label{eq:main-kernel-new}
\end{eqnarray}
where $\tilde{X}=f(X;\bm{\Theta_{FE}})$, and the first and second terms are $\frac{1}{n^2}\left\|\bm \Theta_{Enc} \bm K_{\tilde{X}} \bm H \bm L_{Y} \right\|^2_F$ and $\frac{1}{n^2}\left\|\bm \Theta_{Enc} \bm K_{\tilde{X}_c} \bm H \bm L_{S_c} \right\|^2_F$, respectively. The parameters $\bm{\Theta}_{Enc}$ can now be solved exactly via a closed-form solution:
\begin{theorem}
\label{thm:main-emp}
A global optimizer of \eqref{eq:main-kernel-new} is 
\begin{eqnarray}
\bm f^{\text{opt}}_{\mathcal H_{\tilde{X}}}(\tilde{X}; \bm \Theta_{Enc}) =
\bm \Theta^{\text{opt}}_{Enc}
\left[k_{\tilde{X}}(\tilde{\bm{x}}_1, {\tilde{X}}),\cdots, k_{\tilde{X}}(\tilde{\bm{x}}_n, {\tilde{X}})\right]^T\nn
\end{eqnarray}
where $\bm \Theta^{\text{opt}}_{Enc}=\bm U^T \bm L_{\tilde{X}}^\dagger\in \mathbb R^{r\times n}$ and the columns of $\bm U$ are eigenvectors corresponding to the $r$ largest eigenvalues of the following generalized eigenvalue problem.
\begin{eqnarray}\label{eq:eig-emp}
\left((1-\lambda) \bm L^T_{\tilde{X}} \bm H\bm K_Y\bm H \bm L_{\tilde{X}}  -\lambda \bm L^T_{{\tilde{X}}_c} \bm H\bm K_{S_c}\bm H \bm L_{{\tilde{X}}_c} \right)\bm u \nn\\
= \lambda \left(\frac{1}{n}\,\bm L^T_{\tilde{X}} \bm H \bm L_{\tilde{X}} + \gamma \bm I\right) \bm u.
\end{eqnarray}
Here $\bm L_{\tilde{X}}\bm L_{\tilde{X}}^T=\bm K_{\tilde{X}}$, ${\tilde{X}}_c \sim p(\tilde{X}|Y=y)$ and $S_c \sim p(S|Y=y)$.
\end{theorem}
\begin{proof}
The objective in~\eqref{eq:main-kernel-new} reduces to a generalized eigenvalue problem~\cite{kokiopoulou2011trace} by expressing it as a trace optimization problem. See supplementary for detailed proof.
\end{proof}
While this is a general solution to \eqref{eq:main-kernel-new}, the solution for each group fairness case is detailed in the supplementary.

\vspace{3pt}
\noindent\textbf{Alternating Optimization:} Now we present our full algorithm to optimize \eqref{eq:main-kernel}. We adopt standard minibatch to learn the feature encoder's parameters $\bm{\Theta}_{FE}$ and the closed-form solver for the fair encoder parameters $\bm{\Theta}_{Enc}$. We optimize them alternatively where in each iteration, we update $\bm{\Theta}_{Enc}$ while freezing $\bm{\Theta}_{FE}$ and vice-versa. Specifically, to optimize the fair encoder's parameters $\bm{\Theta}_{Enc}$, we extract features from the data using the frozen feature extractor and use the closed-form solution in \eqref{eq:main-kernel-new} to update $\bm{\Theta}_{Enc}$. Then, we update the feature extractor's parameters $\bm{\Theta}_{FE}$ through minibatch SGD in \eqref{eq:main-kernel}  while freezing the encoder parameters. We repeat this process for every minibatch iteration. More details and an illustration of this alternating algorithm can be found in the supplementary material.
\subsection{Numerically Estimating the LST\label{sec:L}}
The Label Space Trade-off (LST) arises when the representation $Z$ is not restricted to be a function of $X$. Following the discussion in the previous subsection, this trade-off can be formulated as,
\begin{eqnarray}\label{eq:trZ}
 \sup_{Z\in L_r^2} \Big\{(1-\lambda)\, \text{Dep}\big(Z, Y\big)-\lambda\,\text{Dep}\big( Z, S | Y = y \big)\Big\},
\end{eqnarray}
where $L_r^2$ is the space of all RVs of dimension $r$ with finite variance, i.e., $\mathbb E_Z \big[ \big\|Z-\mathbb E[Z]\big\|^2<\infty\big]$. From~\eqref{eq:trZ}, observe that the optimal $Z$ is a function of $\bm p_{Y, S}$ only. Therefore, instead of directly optimizing $Z$ over $L_r^2$, equivalently, we optimize for $\bm \Theta_{FE}$ and $\bm \Theta_{Enc}$ as
\begin{eqnarray}\label{eq:compos-inherit}
\max_{\bm \Theta_{FE}, \bm \Theta_{Enc}} \Big\{(1-\lambda)\, \text{Dep}\left(\bm f\left(Y, S; \bm \Theta_{FE}, \bm \Theta_{Enc}\right), Y\right)\nn\\
+\lambda\, \text{Dep}\left(\bm f\left(Y,S; \bm \Theta_{FE}, \bm \Theta_{Enc}\right),  S | Y=y \right) \Big\}.
\end{eqnarray}
Here $\bm{f}$ is a function of the labels $Y$ and $S$, i.e., the model takes as input $Y$ and $S$ and seeks to remove the information corresponding to $S$, including that present in $Y$. In practice, to improve the stability of the optimization and facilitate learning, in addition to $Y$ and $S$, we also use $X$.
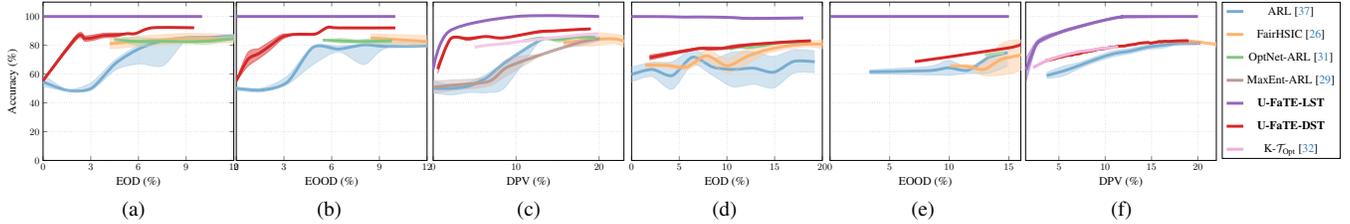
\begin{figure*}[!t]
    \def\subf{0.14}
    \def\scb {0.37}
    \centering
    \begin{adjustwidth*}{0em}{0em}
    \captionsetup[subfigure]{oneside,margin={3em,0em}}
    \begin{subfigure}[c]{\subf\linewidth}
        \centering
        \scalebox{\scb}{
        \input{./plots/celeba-eo-line}
        }
        \caption{}
        \label{fig:results:celeba:eo}
    \end{subfigure}
    \hspace{1.1em}
    \captionsetup[subfigure]{oneside,margin={0.7em,0em}}
    \begin{subfigure}[c]{\subf\linewidth}
        \centering
        \scalebox{\scb}{
        \input{./plots/celeba-eoo-line}
        }
        \caption{}
        \label{fig:results:celeba:eoo}
    \end{subfigure}
    \hspace{-0.01em}
    \captionsetup[subfigure]{oneside,margin={0.9em,0em}}
    \begin{subfigure}[c]{\subf\linewidth}
        \centering
        \scalebox{\scb}{
        \input{./plots/celeba-dpv-line}
        }
        \caption{}
        \label{fig:results:celeba:dpv}
    \end{subfigure}
    \hspace{0.05em}
    \captionsetup[subfigure]{oneside,margin={0.7em,0em}}
    \begin{subfigure}[c]{\subf\linewidth}
        \centering
        \scalebox{\scb}{
        \input{./plots/folk-eo-line}
        }
        \caption{}
        \label{fig:results:folktable:eo}
    \end{subfigure}
    \hspace{0.05em}
    \captionsetup[subfigure]{oneside,margin={0.7em,0em}}
    \begin{subfigure}[c]{\subf\linewidth}
        \centering
        \scalebox{\scb}{
        \input{./plots/folk-eoo-line}
        }
        \caption{}
        \label{fig:results:folktable:eoo}
    \end{subfigure}
    \hspace{0.15em}
    \captionsetup[subfigure]{oneside,margin={0.5em,0em}}
    \begin{subfigure}[c]{\subf\linewidth}
        \centering
        \scalebox{\scb}{
        \input{./plots/folk-dpv-line}
        }
        \caption{}
        \label{fig:results:folktable:dpv}
    \end{subfigure}
    \end{adjustwidth*}
    \caption{\textbf{Evaluating Fair Representation Learning Methods:} Accuracy versus fairness trade-offs on CelebA (a)-(c) and FolkTable (d)-(f). (a) and (d) show the trade-off for Equalized Opportunity as the fairness constraint. (b) and (e) show the trade-off for Equality of Odds as the fairness constraint, and (c) and (f) show the trade-off for Demographic Parity as the fairness constraint. The solid lines represent the mean accuracy at a given fairness value, and the shaded region shows the uncertainty of the trade-off. Both DST and LST estimates from \methodName\ are stable. Among the FRL methods, K-$\mathcal T_{\text{Opt}}$ is closest to the DST, while ARL has the most variance.\label{fig:results:eo-eoo-dpv}}
\end{figure*}
\section{Scheme for Evaluating Representations\label{sec:metrics}}
The primary utility of the trade-offs is in illuminating the fundamental limits of learning algorithms in mitigating unfairness and evaluating the effectiveness of learned representations w.r.t. utility and fairness. This includes representations that provide a single solution in the utility-fairness plane and multiple solutions that span the trade-off between utility and fairness. Fairness evaluations in prior literature focused primarily on relative comparisons of the models to each other. Such a comparison, however, precludes an understanding of how far the solution is from the inherent limits of the task. Elucidating and numerically quantifying the inherent trade-offs facilitates such an understanding and drives further algorithmic development.

A standard way to evaluate bi-objective solutions is to plot them on a 2-D plane, identify non-dominated solutions, or compare their dominance w.r.t. each other. 
More details can be found in the supplementary material.

%% file: plots/celeba-eo-line.tex
\begin{tikzpicture}[scale=1.0]
\begin{axis}
        [ xlabel=\large{EOD (\%)}, 
        ylabel=\large{Accuracy (\%)},  
        xmin=-0.0, xmax=12, ymin=0.0, ymax=110, grid=major, grid style=dotted, legend style={at={(.99,0.01)},anchor=south east},
        xtick distance=3, ylabel near ticks,
        yticklabel style={/pgf/number format/fixed}, xticklabel style={/pgf/number format/fixed}]

        \definecolor{color1}{HTML}{1f77b4} %
        \definecolor{color2}{HTML}{ff7f0e} %
        \definecolor{color3}{HTML}{2ca02c} %
        \definecolor{color4}{HTML}{d62728} %
        \definecolor{color5}{HTML}{9467bd} %
        \definecolor{color6}{HTML}{8c564b} %
        \definecolor{color7}{HTML}{e377c2} %
        \definecolor{color8}{HTML}{7f7f7f} %
        \definecolor{color9}{HTML}{bcbd22} %
        \definecolor{color10}{HTML}{17becf} %

        \pgfplotstableread{plots/txt/celeba-eo-line/ARL/val_ctl_EO_m_var_EO___val_tgt_utility.txt}{\data}
        
        \addplot[color1!60, line width=3pt, smooth] table[x expr={\thisrow{x}*100}, y expr={\thisrow{mean}*100}] {\data};
        \addlegendentry{ARL\cite{NIPS2017_8cb22bdd}}
        
        \addplot[color1!40, name path=c1, opacity=0.5, smooth, forget plot] table[x expr={\thisrow{x}*100}, y expr=100*\thisrow{mean}+100*\thisrow{std_u}] {\data};
        \addplot[color1!40, name path=c2, opacity=0.5, smooth, forget plot] table[x expr={\thisrow{x}*100}, y expr=100*\thisrow{mean}-100*\thisrow{std_l}] {\data};
        \addplot[color1!40, opacity=0.5, forget plot] fill between[of=c1 and c2];

        \pgfplotstableread{plots/txt/celeba-eo-line/HSIC/val_ctl_EO_m_var_EO___val_tgt_utility.txt}{\data}
        
        \addplot[color2!60, line width=3pt, smooth] table[x expr={\thisrow{x}*100}, y expr={\thisrow{mean}*100}] {\data};
        \addlegendentry{HSIC-IRepL\cite{Quadrianto_2019_CVPR}}
        
        \addplot[color2!40, name path=c3, opacity=0.5, smooth, forget plot] table[x expr={\thisrow{x}*100}, y expr=100*\thisrow{mean}+100*\thisrow{std}] {\data};
        \addplot[color2!40, name path=c4, opacity=0.5, smooth, forget plot] table[x expr={\thisrow{x}*100}, y expr=100*\thisrow{mean}-100*\thisrow{std}] {\data};
        \addplot[color2!40, opacity=0.5, forget plot] fill between[of=c3 and c4];

        \pgfplotstableread{plots/txt/celeba-eo-line/OptNet/val_ctl_EO_m_var_EO___val_tgt_utility.txt}{\data}
        
        \addplot[color3!60, line width=3pt, smooth] table[x expr={\thisrow{x}*100}, y expr={\thisrow{mean}*100}] {\data};
        \addlegendentry{OptNet\cite{10.1007/978-3-030-86520-7_45}}
        
        \addplot[color3!40, name path=c5, opacity=0.5, smooth, forget plot] table[x expr={\thisrow{x}*100}, y expr=100*\thisrow{mean}+100*\thisrow{std}] {\data};
        \addplot[color3!40, name path=c6, opacity=0.5, smooth, forget plot] table[x expr={\thisrow{x}*100}, y expr=100*\thisrow{mean}-100*\thisrow{std}] {\data};
        \addplot[color3!40, opacity=0.5, forget plot] fill between[of=c5 and c6];

        \pgfplotstableread{plots/txt/celeba-eo-line/LST-XYS/val_ctl_EO_m_var_EO___val_tgt_utility.txt}{\data}
        
        \addplot[color5, line width=3pt, smooth] table[x expr={\thisrow{x}*100}, y expr={\thisrow{mean}*100}] {\data};
        \addlegendentry{\methodName-LST}
        
        \addplot[color5, name path=c5, opacity=0.5, smooth, forget plot] table[x expr={\thisrow{x}*100}, y expr=100*\thisrow{mean}+100*\thisrow{std}] {\data};
        \addplot[color5, name path=c6, opacity=0.5, smooth, forget plot] table[x expr={\thisrow{x}*100}, y expr=100*\thisrow{mean}-100*\thisrow{std}] {\data};
        \addplot[color5, opacity=0.5, forget plot] fill between[of=c5 and c6];
        
        \pgfplotstableread{plots/txt/celeba-eo-line/DST/val_ctl_EO_m_var_EO___val_tgt_utility.txt}{\data}
        
        \addplot[color4, line width=3pt, smooth] table[x expr={\thisrow{x}*100}, y expr={\thisrow{mean}*100}] {\data};
        \addlegendentry{\methodName-DST}
        
        \addplot[color4, name path=c7, opacity=0.5, smooth, forget plot] table[x expr={\thisrow{x}*100}, y expr=100*\thisrow{mean}+100*\thisrow{std}] {\data};
        \addplot[color4, name path=c8, opacity=0.5, smooth, forget plot] table[x expr={\thisrow{x}*100}, y expr=100*\thisrow{mean}-100*\thisrow{std}] {\data};
        \addplot[color4, opacity=0.5, forget plot] fill between[of=c7 and c8];

\legend{}
\end{axis}

\end{tikzpicture}

%% file: plots/celeba-eoo-line.tex
\begin{tikzpicture}[scale=1.0]
\begin{axis}
        [ xlabel=\large{EOOD (\%)}, 
        ymajorticks=false,
        xmin=-0.0, xmax=12, ymin=0.0, ymax=110, grid=major, grid style=dotted, legend style={at={(.99,0.01)},anchor=south east},
        xtick distance=3, ylabel near ticks,
        yticklabel style={/pgf/number format/fixed}, xticklabel style={/pgf/number format/fixed}]

        \definecolor{color1}{HTML}{1f77b4} %
        \definecolor{color2}{HTML}{ff7f0e} %
        \definecolor{color3}{HTML}{2ca02c} %
        \definecolor{color4}{HTML}{d62728} %
        \definecolor{color5}{HTML}{9467bd} %
        \definecolor{color6}{HTML}{8c564b} %
        \definecolor{color7}{HTML}{e377c2} %
        \definecolor{color8}{HTML}{7f7f7f} %
        \definecolor{color9}{HTML}{bcbd22} %
        \definecolor{color10}{HTML}{17becf} %

        \pgfplotstableread{plots/txt/celeba-eoo-line/ARL/val_ctl_EOO_m_var_EOO___val_tgt_utility.txt}{\data}
        
        \addplot[color1!60, line width=3pt, smooth] table[x expr={\thisrow{x}*100}, y expr={\thisrow{mean}*100}] {\data};
        \addlegendentry{ARL\cite{NIPS2017_8cb22bdd}}
        
        \addplot[color1!40, name path=c1, opacity=0.5, smooth, forget plot] table[x expr={\thisrow{x}*100}, y expr=100*\thisrow{mean}+100*\thisrow{std_u}] {\data};
        \addplot[color1!40, name path=c2, opacity=0.5, smooth, forget plot] table[x expr={\thisrow{x}*100}, y expr=100*\thisrow{mean}-100*\thisrow{std_l}] {\data};
        \addplot[color1!40, opacity=0.5, forget plot] fill between[of=c1 and c2];

        \pgfplotstableread{plots/txt/celeba-eoo-line/HSIC/val_ctl_EOO_m_var_EOO___val_tgt_utility.txt}{\data}
        
        \addplot[color2!60, line width=3pt, smooth] table[x expr={\thisrow{x}*100}, y expr={\thisrow{mean}*100}] {\data};
        \addlegendentry{HSIC-FRepL\cite{Quadrianto_2019_CVPR}}
        
        \addplot[color2!40, name path=c3, opacity=0.5, smooth, forget plot] table[x expr={\thisrow{x}*100}, y expr=100*\thisrow{mean}+100*\thisrow{std}] {\data};
        \addplot[color2!40, name path=c4, opacity=0.5, smooth, forget plot] table[x expr={\thisrow{x}*100}, y expr=100*\thisrow{mean}-100*\thisrow{std}] {\data};
        \addplot[color2!40, opacity=0.5, forget plot] fill between[of=c3 and c4];

        \pgfplotstableread{plots/txt/celeba-eoo-line/OptNet/val_ctl_EOO_m_var_EOO___val_tgt_utility.txt}{\data}
        
        \addplot[color3!60, line width=3pt, smooth] table[x expr={\thisrow{x}*100}, y expr={\thisrow{mean}*100}] {\data};
        \addlegendentry{OptNet\cite{10.1007/978-3-030-86520-7_45}}
        
        \addplot[color3!40, name path=c5, opacity=0.5, smooth, forget plot] table[x expr={\thisrow{x}*100}, y expr=100*\thisrow{mean}+100*\thisrow{std}] {\data};
        \addplot[color3!40, name path=c6, opacity=0.5, smooth, forget plot] table[x expr={\thisrow{x}*100}, y expr=100*\thisrow{mean}-100*\thisrow{std}] {\data};
        \addplot[color3!40, opacity=0.5, forget plot] fill between[of=c5 and c6];

        \pgfplotstableread{plots/txt/celeba-eoo-line/Kernel-XYS/val_ctl_EOO_m_var_EOO___val_tgt_utility.txt}{\data}
        
        \addplot[color5, line width=3pt, smooth] table[x expr={\thisrow{x}*100}, y expr={\thisrow{mean}*100}] {\data};
        \addlegendentry{\methodName-LST}
        
        \addplot[color5, name path=c5, opacity=0.5, smooth, forget plot] table[x expr={\thisrow{x}*100}, y expr=100*\thisrow{mean}+100*\thisrow{std}] {\data};
        \addplot[color5, name path=c6, opacity=0.5, smooth, forget plot] table[x expr={\thisrow{x}*100}, y expr=100*\thisrow{mean}-100*\thisrow{std}] {\data};
        \addplot[color5, opacity=0.5, forget plot] fill between[of=c5 and c6];
        
        \pgfplotstableread{plots/txt/celeba-eoo-line/DST/val_ctl_EOO_m_var_EOO___val_tgt_utility.txt}{\data}
        
        \addplot[color4, line width=3pt, smooth] table[x expr={\thisrow{x}*100}, y expr={\thisrow{mean}*100}] {\data};
        \addlegendentry{\methodName-DST}
        
        \addplot[color4, name path=c7, opacity=0.5, smooth, forget plot] table[x expr={\thisrow{x}*100}, y expr=100*\thisrow{mean}+100*\thisrow{std}] {\data};
        \addplot[color4, name path=c8, opacity=0.5, smooth, forget plot] table[x expr={\thisrow{x}*100}, y expr=100*\thisrow{mean}-100*\thisrow{std}] {\data};
        \addplot[color4, opacity=0.5, forget plot] fill between[of=c7 and c8];

    \legend{}
\end{axis}
\end{tikzpicture}

%% file: plots/celeba-dpv-line.tex
\begin{tikzpicture}[scale=1.0]
\begin{axis}
        [ xlabel=\large{DPV (\%)}, 
        ymajorticks=false,
        xmin=-0.0, xmax=23.0, ymin=0.0, ymax=110, grid=major, grid style=dotted,  legend style={at={(1.05,1)},anchor=north west},
        xtick distance=10, ylabel near ticks,
        yticklabel style={/pgf/number format/fixed}, xticklabel style={/pgf/number format/fixed}]

        \definecolor{color1}{HTML}{1f77b4} %
        \definecolor{color2}{HTML}{ff7f0e} %
        \definecolor{color3}{HTML}{2ca02c} %
        \definecolor{color4}{HTML}{d62728} %
        \definecolor{color5}{HTML}{9467bd} %
        \definecolor{color6}{HTML}{8c564b} %
        \definecolor{color7}{HTML}{e377c2} %
        \definecolor{color8}{HTML}{7f7f7f} %
        \definecolor{color9}{HTML}{bcbd22} %
        \definecolor{color10}{HTML}{17becf} %

        \pgfplotstableread{plots/txt/celeba-dpv-line/ARL/val_ctl_SP_m_var_SP___val_tgt_utility.txt}{\data}
        
        \addplot[color1!60, line width=3pt, smooth] table[x expr={\thisrow{x}*100}, y expr={\thisrow{mean}*100}] {\data};
        \addlegendentry{ARL\cite{NIPS2017_8cb22bdd}}
        
        \addplot[color1!40, name path=c1, opacity=0.5, smooth, forget plot] table[x expr={\thisrow{x}*100}, y expr=100*\thisrow{mean}+100*\thisrow{std_u}] {\data};
        \addplot[color1!40, name path=c2, opacity=0.5, smooth, forget plot] table[x expr={\thisrow{x}*100}, y expr=100*\thisrow{mean}-100*\thisrow{std_l}] {\data};
        \addplot[color1!40, opacity=0.5, forget plot] fill between[of=c1 and c2];

        \pgfplotstableread{plots/txt/celeba-dpv-line/HSIC/val_ctl_SP_m_var_SP___val_tgt_utility.txt}{\data}
        
        \addplot[color2!60, line width=3pt, smooth] table[x expr={\thisrow{x}*100}, y expr={\thisrow{mean}*100}] {\data};
        \addlegendentry{HSIC-FRepL\cite{Quadrianto_2019_CVPR}}
        
        \addplot[color2!40, name path=c3, opacity=0.5, smooth, forget plot] table[x expr={\thisrow{x}*100}, y expr=100*\thisrow{mean}+100*\thisrow{std}] {\data};
        \addplot[color2!40, name path=c4, opacity=0.5, smooth, forget plot] table[x expr={\thisrow{x}*100}, y expr=100*\thisrow{mean}-100*\thisrow{std}] {\data};
        \addplot[color2!40, opacity=0.5, forget plot] fill between[of=c3 and c4];

        \pgfplotstableread{plots/txt/celeba-dpv-line/MaxEnt/val_ctl_SP_m_var_SP___val_tgt_utility.txt}{\data}
        
        \addplot[color6!60, line width=3pt, smooth] table[x expr={\thisrow{x}*100}, y expr={\thisrow{mean}*100}] {\data};
        \addlegendentry{MaxEnt-ARL\cite{roy2019mitigating}}
        
        \addplot[color6!40, name path=c5, opacity=0.5, smooth, forget plot] table[x expr={\thisrow{x}*100}, y expr=100*\thisrow{mean}+100*\thisrow{std}] {\data};
        \addplot[color6!40, name path=c6, opacity=0.5, smooth, forget plot] table[x expr={\thisrow{x}*100}, y expr=100*\thisrow{mean}-100*\thisrow{std}] {\data};
        \addplot[color6!40, opacity=0.5, forget plot] fill between[of=c5 and c6];

        \pgfplotstableread{plots/txt/celeba-dpv-line/OptNet/val_ctl_SP_m_var_SP___val_tgt_utility.txt}{\data}
        
        \addplot[color3!60, line width=3pt, smooth] table[x expr={\thisrow{x}*100}, y expr={\thisrow{mean}*100}] {\data};
        \addlegendentry{OptNet\cite{10.1007/978-3-030-86520-7_45}}
        
        \addplot[color3!40, name path=c5, opacity=0.5, smooth, forget plot] table[x expr={\thisrow{x}*100}, y expr=100*\thisrow{mean}+100*\thisrow{std}] {\data};
        \addplot[color3!40, name path=c6, opacity=0.5, smooth, forget plot] table[x expr={\thisrow{x}*100}, y expr=100*\thisrow{mean}-100*\thisrow{std}] {\data};
        \addplot[color3!40, opacity=0.5, forget plot] fill between[of=c5 and c6];

        \pgfplotstableread{plots/txt/celeba-dpv-line/LST/val_ctl_SP_m_var_SP___val_tgt_utility.txt}{\data}
        
        \addplot[color5, line width=3pt, smooth] table[x expr={\thisrow{x}*100}, y expr={\thisrow{mean}*100}] {\data};
        \addlegendentry{\textbf{\methodName-LST}}
        
        \addplot[color5, name path=c5, opacity=0.5, smooth, forget plot] table[x expr={\thisrow{x}*100}, y expr=100*\thisrow{mean}+100*\thisrow{std}] {\data};
        \addplot[color5, name path=c6, opacity=0.5, smooth, forget plot] table[x expr={\thisrow{x}*100}, y expr=100*\thisrow{mean}-100*\thisrow{std}] {\data};
        \addplot[color5, opacity=0.5, forget plot] fill between[of=c5 and c6];
        
        \pgfplotstableread{plots/txt/celeba-dpv-line/DST/val_ctl_SP_m_var_SP___val_tgt_utility.txt}{\data}
        
        \addplot[color4, line width=3pt, smooth] table[x expr={\thisrow{x}*100}, y expr={\thisrow{mean}*100}] {\data};
        \addlegendentry{\textbf{\methodName-DST}}
        
        \addplot[color4, name path=c7, opacity=0.5, smooth, forget plot] table[x expr={\thisrow{x}*100}, y expr=100*\thisrow{mean}+100*\thisrow{std}] {\data};
        \addplot[color4, name path=c8, opacity=0.5, smooth, forget plot] table[x expr={\thisrow{x}*100}, y expr=100*\thisrow{mean}-100*\thisrow{std}] {\data};
        \addplot[color4, opacity=0.5, forget plot] fill between[of=c7 and c8];
        
        \pgfplotstableread{plots/txt/celeba-dpv-line/TMLR/val_ctl_SP_m_var_SP___val_tgt_utility.txt}{\data}
        
        \addplot[color7!40, line width=3pt, smooth] table[x expr={\thisrow{x}*100}, y expr={\thisrow{mean}*100}] {\data};
        \addlegendentry{K-$\mathcal T_{\text{Opt}}$\cite{sadeghi2022on}}
        
        \addplot[color7!40, name path=c5, opacity=0.5, smooth, forget plot] table[x expr={\thisrow{x}*100}, y expr=100*\thisrow{mean}+100*\thisrow{std}] {\data};
        \addplot[color7!40, name path=c6, opacity=0.5, smooth, forget plot] table[x expr={\thisrow{x}*100}, y expr=100*\thisrow{mean}-100*\thisrow{std}] {\data};
        \addplot[color7!40, opacity=0.5, forget plot] fill between[of=c5 and c6];

    \legend{}
\end{axis}

\end{tikzpicture}

%% file: plots/folk-eo-line.tex
\begin{tikzpicture}[scale=1.0]
\begin{axis}
        [ xlabel=\large{EOD (\%)}, 
        ymajorticks=false, 
        xmin=-0.0, xmax=20, ymin=0.0, ymax=110, grid=major, grid style=dotted, legend style={at={(.99,0.01)},anchor=south east},
        xtick distance=5, ylabel near ticks,
        yticklabel style={/pgf/number format/fixed}, xticklabel style={/pgf/number format/fixed}]

        \definecolor{color1}{HTML}{1f77b4} %
        \definecolor{color2}{HTML}{ff7f0e} %
        \definecolor{color3}{HTML}{2ca02c} %
        \definecolor{color4}{HTML}{d62728} %
        \definecolor{color5}{HTML}{9467bd} %
        \definecolor{color6}{HTML}{8c564b} %
        \definecolor{color7}{HTML}{e377c2} %
        \definecolor{color8}{HTML}{7f7f7f} %
        \definecolor{color9}{HTML}{bcbd22} %
        \definecolor{color10}{HTML}{17becf} %

        \pgfplotstableread{plots/txt/folk-eo-line/ARL/val_ctl_EO_m_var_EO___val_tgt_utility.txt}{\data}
        
        \addplot[color1!60, line width=3pt, smooth] table[x expr={\thisrow{x}*100}, y expr={\thisrow{mean}*100}] {\data};
        \addlegendentry{ARL\cite{NIPS2017_8cb22bdd}}
        
        \addplot[color1!40, name path=c1, opacity=0.5, smooth, forget plot] table[x expr={\thisrow{x}*100}, y expr=100*\thisrow{mean}+100*\thisrow{std}] {\data};
        \addplot[color1!40, name path=c2, opacity=0.5, smooth, forget plot] table[x expr={\thisrow{x}*100}, y expr=100*\thisrow{mean}-100*\thisrow{std}] {\data};
        \addplot[color1!40, opacity=0.5, forget plot] fill between[of=c1 and c2];

        \pgfplotstableread{plots/txt/folk-eo-line/HSIC/val_ctl_EO_m_var_EO___val_tgt_utility.txt}{\data}
        
        \addplot[color2!60, line width=3pt, smooth] table[x expr={\thisrow{x}*100}, y expr={\thisrow{mean}*100}] {\data};
        \addlegendentry{HSIC-FRepL\cite{Quadrianto_2019_CVPR}}
        
        \addplot[color2!40, name path=c3, opacity=0.5, smooth, forget plot] table[x expr={\thisrow{x}*100}, y expr=100*\thisrow{mean}+100*\thisrow{std}] {\data};
        \addplot[color2!40, name path=c4, opacity=0.5, smooth, forget plot] table[x expr={\thisrow{x}*100}, y expr=100*\thisrow{mean}-100*\thisrow{std}] {\data};
        \addplot[color2!40, opacity=0.5, forget plot] fill between[of=c3 and c4];

        \pgfplotstableread{plots/txt/folk-eo-line/OptNet/val_ctl_EO_m_var_EO___val_tgt_utility.txt}{\data}
        
        \addplot[color3!60, line width=3pt, smooth] table[x expr={\thisrow{x}*100}, y expr={\thisrow{mean}*100}] {\data};
        \addlegendentry{OptNet\cite{10.1007/978-3-030-86520-7_45}}
        
        \addplot[color3!40, name path=c5, opacity=0.5, smooth, forget plot] table[x expr={\thisrow{x}*100}, y expr=100*\thisrow{mean}+100*\thisrow{std}] {\data};
        \addplot[color3!40, name path=c6, opacity=0.5, smooth, forget plot] table[x expr={\thisrow{x}*100}, y expr=100*\thisrow{mean}-100*\thisrow{std}] {\data};
        \addplot[color3!40, opacity=0.5, forget plot] fill between[of=c5 and c6];

        \pgfplotstableread{plots/txt/folk-eo-line/Kernel-XYS/val_ctl_EO_m_var_EO___val_tgt_utility.txt}{\data}
        
        \addplot[color5, line width=3pt, smooth] table[x expr={\thisrow{x}*100}, y expr={\thisrow{mean}*100}] {\data};
        \addlegendentry{\methodName-LST}
        
        \addplot[color5, name path=c5, opacity=0.5, smooth, forget plot] table[x expr={\thisrow{x}*100}, y expr=100*\thisrow{mean}+100*\thisrow{std}] {\data};
        \addplot[color5, name path=c6, opacity=0.5, smooth, forget plot] table[x expr={\thisrow{x}*100}, y expr=100*\thisrow{mean}-100*\thisrow{std}] {\data};
        \addplot[color5, opacity=0.5, forget plot] fill between[of=c5 and c6];
        
        \pgfplotstableread{plots/txt/folk-eo-line/DST/val_ctl_EO_m_var_EO___val_tgt_utility.txt}{\data}
        
        \addplot[color4, line width=3pt, smooth] table[x expr={\thisrow{x}*100}, y expr={\thisrow{mean}*100}] {\data};
        \addlegendentry{\methodName-DST}
        
        \addplot[color4, name path=c7, opacity=0.5, smooth, forget plot] table[x expr={\thisrow{x}*100}, y expr=100*\thisrow{mean}+100*\thisrow{std}] {\data};
        \addplot[color4, name path=c8, opacity=0.5, smooth, forget plot] table[x expr={\thisrow{x}*100}, y expr=100*\thisrow{mean}-100*\thisrow{std}] {\data};
        \addplot[color4, opacity=0.5, forget plot] fill between[of=c7 and c8];
    \legend{}
\end{axis}
\end{tikzpicture}

%% file: plots/folk-eoo-line.tex
\begin{tikzpicture}[scale=1.0]
\begin{axis}
        [ xlabel=\large{EOOD (\%)}, 
        ymajorticks=false, 
        xmin=-0.0, xmax=16.0, ymin=0.0, ymax=110, grid=major, grid style=dotted, legend style={at={(.99,0.01)},anchor=south east},
        xtick distance=5, ylabel near ticks,
        yticklabel style={/pgf/number format/fixed}, xticklabel style={/pgf/number format/fixed}]

        \definecolor{color1}{HTML}{1f77b4} %
        \definecolor{color2}{HTML}{ff7f0e} %
        \definecolor{color3}{HTML}{2ca02c} %
        \definecolor{color4}{HTML}{d62728} %
        \definecolor{color5}{HTML}{9467bd} %
        \definecolor{color6}{HTML}{8c564b} %
        \definecolor{color7}{HTML}{e377c2} %
        \definecolor{color8}{HTML}{7f7f7f} %
        \definecolor{color9}{HTML}{bcbd22} %
        \definecolor{color10}{HTML}{17becf} %

        \pgfplotstableread{plots/txt/folk-eoo-line/ARL/val_ctl_EOO_m_var_EOO___val_tgt_utility.txt}{\data}
        
        \addplot[color1!60, line width=3pt, smooth] table[x expr={\thisrow{x}*100}, y expr={\thisrow{mean}*100}] {\data};
        \addlegendentry{ARL~\cite{xie2017controllable}}
        
        \addplot[color1!40, name path=c1, opacity=0.5, smooth, forget plot] table[x expr={\thisrow{x}*100}, y expr=100*\thisrow{mean}+100*\thisrow{std_u}] {\data};
        \addplot[color1!40, name path=c2, opacity=0.5, smooth, forget plot] table[x expr={\thisrow{x}*100}, y expr=100*\thisrow{mean}-100*\thisrow{std_l}] {\data};
        \addplot[color1!40, opacity=0.5, forget plot] fill between[of=c1 and c2];

        \pgfplotstableread{plots/txt/folk-eoo-line/HSIC/val_ctl_EOO_m_var_EOO___val_tgt_utility.txt}{\data}
        
        \addplot[color2!60, line width=3pt, smooth] table[x expr={\thisrow{x}*100}, y expr={\thisrow{mean}*100}] {\data};
        \addlegendentry{FairHSIC~\cite{quadrianto2019discovering}}
        
        \addplot[color2!40, name path=c3, opacity=0.5, smooth, forget plot] table[x expr={\thisrow{x}*100}, y expr=100*\thisrow{mean}+100*\thisrow{std}] {\data};
        \addplot[color2!40, name path=c4, opacity=0.5, smooth, forget plot] table[x expr={\thisrow{x}*100}, y expr=100*\thisrow{mean}-100*\thisrow{std}] {\data};
        \addplot[color2!40, opacity=0.5, forget plot] fill between[of=c3 and c4];

        \pgfplotstableread{plots/txt/folk-eoo-line/Kernel-XYS/val_ctl_EOO_m_var_EOO___val_tgt_utility.txt}{\data}
        
        \addplot[color5, line width=3pt, smooth] table[x expr={\thisrow{x}*100}, y expr={\thisrow{mean}*100}] {\data};
        \addlegendentry{\methodName-LST}
        
        \addplot[color5, name path=c5, opacity=0.5, smooth, forget plot] table[x expr={\thisrow{x}*100}, y expr=100*\thisrow{mean}-100*\thisrow{std}] {\data};
        \addplot[color5, name path=c6, opacity=0.5, smooth, forget plot] table[x expr={\thisrow{x}*100}, y expr=100*\thisrow{mean}+100*\thisrow{std}] {\data};
        \addplot[color5, opacity=0.5, forget plot] fill between[of=c5 and c6];

        \pgfplotstableread{plots/txt/folk-eoo-line/OptNet/val_ctl_EOO_m_var_EOO___val_tgt_utility.txt}{\data}
        
        \addplot[color3!40, line width=3pt, smooth] table[x expr={\thisrow{x}*100}, y expr={\thisrow{mean}*100}] {\data};
        \addlegendentry{OptNet-ARL~\cite{sadeghi2021adversarial}}
        
        \addplot[color3!40, name path=c5, opacity=0.5, smooth, forget plot] table[x expr={\thisrow{x}*100}, y expr=100*\thisrow{mean}+100*\thisrow{std}] {\data};;
        \addplot[color3!40, name path=c6, opacity=0.5, smooth, forget plot] table[x expr={\thisrow{x}*100}, y expr=100*\thisrow{mean}-100*\thisrow{std}] {\data};;
        \addplot[color3!40, opacity=0.5, forget plot] fill between[of=c5 and c6];
        
        \pgfplotstableread{plots/txt/folk-eoo-line/DST/val_ctl_EOO_m_var_EOO___val_tgt_utility.txt}{\data}
        
        \addplot[color4, line width=3pt, smooth] table[x expr={\thisrow{x}*100}, y expr={\thisrow{mean}*100}] {\data};
        \addlegendentry{\methodName-DST}
        
        \addplot[color4, name path=c7, opacity=0.5, smooth, forget plot] table[x expr={\thisrow{x}*100}, y expr=100*\thisrow{mean}+100*\thisrow{std}] {\data};;
        \addplot[color4, name path=c8, opacity=0.5, smooth, forget plot] table[x expr={\thisrow{x}*100}, y expr=100*\thisrow{mean}-100*\thisrow{std}] {\data};;
        \addplot[color4, opacity=0.5, forget plot] fill between[of=c7 and c8];
        
    \legend{}
\end{axis}
\end{tikzpicture}

%% file: plots/folk-dpv-line.tex
\begin{tikzpicture}[scale=1.0]
\begin{axis}
        [ xlabel=\large{DPV (\%)}, 
        ymajorticks=false,
        legend style={nodes={scale=1.2, transform shape}, /tikz/every even column/.append style={row sep=0.305cm}},
        xmin=1.4, xmax=22, ymin=0.0, ymax=110, grid=major, grid style=dotted, legend style={at={(1.03,1)},anchor=north west},
        xtick distance=5, ylabel near ticks,
        yticklabel style={/pgf/number format/fixed}, xticklabel style={/pgf/number format/fixed}]

        \definecolor{color1}{HTML}{1f77b4} %
        \definecolor{color2}{HTML}{ff7f0e} %
        \definecolor{color3}{HTML}{2ca02c} %
        \definecolor{color4}{HTML}{d62728} %
        \definecolor{color5}{HTML}{9467bd} %
        \definecolor{color6}{HTML}{8c564b} %
        \definecolor{color7}{HTML}{e377c2} %
        \definecolor{color8}{HTML}{7f7f7f} %
        \definecolor{color9}{HTML}{bcbd22} %
        \definecolor{color10}{HTML}{17becf} %

        \pgfplotstableread{plots/txt/folk-dpv-line/ARL/val_ctl_SP_m_var_SP___val_tgt_utility.txt}{\data}
        
        \addplot[color1!60, line width=3pt, smooth] table[x expr={\thisrow{x}*100}, y expr={\thisrow{mean}*100}] {\data};
        \addlegendentry{ARL~\cite{xie2017controllable}}
        
        \addplot[color1!40, name path=c1, opacity=0.5, smooth, forget plot] table[x expr={\thisrow{x}*100}, y expr=100*\thisrow{mean}+100*\thisrow{std}] {\data};
        \addplot[color1!40, name path=c2, opacity=0.5, smooth, forget plot] table[x expr={\thisrow{x}*100}, y expr=100*\thisrow{mean}-100*\thisrow{std}] {\data};
        \addplot[color1!40, opacity=0.5, forget plot] fill between[of=c1 and c2];

        \pgfplotstableread{plots/txt/folk-dpv-line/HSIC/val_ctl_SP_m_var_SP___val_tgt_utility.txt}{\data}
        
        \addplot[color2!60, line width=3pt, smooth] table[x expr={\thisrow{x}*100}, y expr={\thisrow{mean}*100}] {\data};
        \addlegendentry{FairHSIC~\cite{quadrianto2019discovering}}
        
        \addplot[color2!40, name path=c3, opacity=0.5, smooth, forget plot] table[x expr={\thisrow{x}*100}, y expr=100*\thisrow{mean}+100*\thisrow{std}] {\data};
        \addplot[color2!40, name path=c4, opacity=0.5, smooth, forget plot] table[x expr={\thisrow{x}*100}, y expr=100*\thisrow{mean}-100*\thisrow{std}] {\data};
        \addplot[color2!40, opacity=0.5, forget plot] fill between[of=c3 and c4];

        \pgfplotstableread{plots/txt/folk-dpv-line/OptNet/val_ctl_SP_m_var_SP___val_tgt_utility.txt}{\data}
        
        \addplot[color3!60, line width=3pt, smooth] table[x expr={\thisrow{x}*100}, y expr={\thisrow{mean}*100}] {\data};
        \addlegendentry{OptNet-ARL~\cite{sadeghi2021adversarial}}
        
        \addplot[color3!40, name path=c5, opacity=0.5, smooth, forget plot] table[x expr={\thisrow{x}*100}, y expr=100*\thisrow{mean}+100*\thisrow{std}] {\data};
        \addplot[color3!40, name path=c6, opacity=0.5, smooth, forget plot] table[x expr={\thisrow{x}*100}, y expr=100*\thisrow{mean}-100*\thisrow{std}] {\data};
        \addplot[color3!40, opacity=0.5, forget plot] fill between[of=c5 and c6];

        \pgfplotstableread{plots/txt/folk-dpv-line/MaxEnt/val_ctl_SP_m_var_SP___val_tgt_utility.txt}{\data}
        
        \addplot[color6!60, line width=3pt, smooth] table[x expr={\thisrow{x}*100}, y expr={\thisrow{mean}*100}] {\data};
        \addlegendentry{MaxEnt-ARL~\cite{roy2019mitigating}}
        
        \addplot[color6!40, name path=c5, opacity=0.5, smooth, forget plot] table[x expr={\thisrow{x}*100}, y expr=100*\thisrow{mean}+100*\thisrow{std}] {\data};
        \addplot[color6!40, name path=c6, opacity=0.5, smooth, forget plot] table[x expr={\thisrow{x}*100}, y expr=100*\thisrow{mean}-100*\thisrow{std}] {\data};
        \addplot[color6!40, opacity=0.5, forget plot] fill between[of=c5 and c6];

        \pgfplotstableread{plots/txt/folk-dpv-line/Kernel-XYS/val_ctl_SP_m_var_SP___val_tgt_utility.txt}{\data}
        
        \addplot[color5, line width=3pt, smooth] table[x expr={\thisrow{x}*100}, y expr={\thisrow{mean}*100}] {\data};
        \addlegendentry{\textbf{\methodName-LST}}
        
        \addplot[color5, name path=c5, opacity=0.5, smooth, forget plot] table[x expr={\thisrow{x}*100}, y expr=100*\thisrow{mean}+100*\thisrow{std}] {\data};
        \addplot[color5, name path=c6, opacity=0.5, smooth, forget plot] table[x expr={\thisrow{x}*100}, y expr=100*\thisrow{mean}-100*\thisrow{std}] {\data};
        \addplot[color5, opacity=0.5, forget plot] fill between[of=c5 and c6];
        
        \pgfplotstableread{plots/txt/folk-dpv-line/DST/val_ctl_SP_m_var_SP___val_tgt_utility.txt}{\data}
        
        \addplot[color4, line width=3pt, smooth] table[x expr={\thisrow{x}*100}, y expr={\thisrow{mean}*100}] {\data};
        \addlegendentry{\textbf{\methodName-DST}}
        
        \addplot[color4, name path=c7, opacity=0.5, smooth, forget plot] table[x expr={\thisrow{x}*100}, y expr=100*\thisrow{mean}+100*\thisrow{std}] {\data};
        \addplot[color4, name path=c8, opacity=0.5, smooth, forget plot] table[x expr={\thisrow{x}*100}, y expr=100*\thisrow{mean}-100*\thisrow{std}] {\data};
        \addplot[color4, opacity=0.5, forget plot] fill between[of=c7 and c8];

        \pgfplotstableread{plots/txt/folk-dpv-line/TMLR/val_ctl_SP_m_var_SP___val_tgt_utility.txt}{\data}
        
        \addplot[color7!60, line width=3pt, smooth] table[x expr={\thisrow{x}*100}, y expr={\thisrow{mean}*100}] {\data};
        \addlegendentry{K-$\mathcal T_{\text{Opt}}$~\cite{sadeghi2022on}}
        
        \addplot[color7!40, name path=c5, opacity=0.5, smooth, forget plot] table[x expr={\thisrow{x}*100}, y expr=100*\thisrow{mean}+100*\thisrow{std}] {\data};
        \addplot[color7!40, name path=c6, opacity=0.5, smooth, forget plot] table[x expr={\thisrow{x}*100}, y expr=100*\thisrow{mean}-100*\thisrow{std}] {\data};
        \addplot[color7!40, opacity=0.5, forget plot] fill between[of=c5 and c6];
        
\end{axis}
\end{tikzpicture}

%% file: 05-experiments.tex
\section{Experimental Evaluation\label{sec:experiments}}

We designed experiments to answer the following:
\begin{enumerate}
    \item How far are existing supervised fair representation learning methods from the two trade-offs? (\S\ref{sec:experiments:results-frl})
    \item How far are zero-shot representations from the two trade-offs? What is the effect of network architecture and pre-training dataset? (\S\ref{sec:experiments:results-vlm})
    \item How far are pre-trained image representations trained in a supervised fashion from the two trade-offs? (\S\ref{sec:experiments:results-supervised})
\end{enumerate}

\subsection{Experimental Setup}

\noindent\textbf{Datasets:\label{sec:experiments:datasets}} We estimate the trade-offs through \methodName\ on an assortment of datasets. 1) \textbf{CelebA~\cite{liu2015deep}} consists of more than 200K face images of celebrities in the wild annotated with 40 binary attributes. 2) \textbf{FairFace~\cite{karkkainen2021fairface}} consists of face images from 7 different race groups labeled with race, sex, and age groups. 3) \textbf{FolkTables~\cite{NEURIPS2021_32e54441}} is a tabular dataset of individuals from fifty states derived from the US Census. 

For experiments on CelebA, the target attribute is \emph{high cheekbones}, and the sensitive attribute is a combination of \emph{sex} and \emph{age} for a total of four classes (young woman, young man, old woman, and old man). For experiments on the FairFace dataset, \emph{sex} (binary) is the target attribute with \emph{race} (7 groups: East Asian, White, Black, Indian, Latino, South Asian, and Middle Eastern) being the sensitive attribute. For the FolkTables dataset, we consider data from Washington State and choose \emph{employment status} (binary) and \emph{age} (discrete value between 1 and 96) of the individuals as the target and sensitive attributes, respectively.

\vspace{3pt}
\noindent\textbf{Metrics:\label{sec:experiments:metrics}}
In all experiments, the utility is measured via classification accuracy. We consider three group fairness metrics (EOD, EOOD, and DPV). The zero-shot CLIP models are evaluated via cosine similarity. Moreover, on the FairFace dataset, we compared the models to the ideal solutions estimated by LST and DST using a weighted normalized Euclidean distance (see supplementary for details). We refer to the distance as $\text{Dist}_{LST}$ and $\text{Dist}_{DST}$, respectively.

\vspace{3pt}
\noindent\textbf{Implementation Details:\label{sec:experiments:implementation-details}} We use ResNet-18 as the feature extractor for \methodName\ and the FRL methods. The final classifier is a two-layer MLP. We evaluate the representations from pre-trained vision models by learning a logistic classifier. To obtain the trade-offs, we run \methodName\, and other FRL methods for multiple values of $\lambda$ between zero and one, where zero corresponds to no fairness constraint, and one corresponds to only fairness.

\vspace{3pt}
\noindent\textbf{Optimizing \methodName{}:} The trade-offs are defined through the dependence terms in \eqref{eq:dep-pop} which involves an expectation over the joint distribution $p(X,Y,S)$. However, due to practical considerations, we only have access to a finite set of samples to estimate the trade-offs. Therefore, we estimate the trade-offs using all the samples available in each dataset without splitting it into train, validation, and test sets. This choice ensures that the estimates account for any possible generalization gap between the train and test distributions and identify the best achievable utility-fairness trade-offs.

\subsection{Evaluating FRL Methods\label{sec:experiments:results-frl}}

\noindent\textbf{FRL Baselines:\label{secf:experiments:frl-baselines}} We consider a wide range of FRL methods based on adversarial learning (ARL~\cite{xie2017controllable} and MaxEnt-ARL~\cite{roy2019mitigating}), dependence measures (FairHSIC~\cite{quadrianto2019discovering}, OptNet-ARL \cite{sadeghi2021adversarial}), and closed-form solvers (K-$\mathcal T_{\text{Opt}}$~\cite{sadeghi2022on}).

\vspace{3pt}
\noindent\textbf{Results:\label{secf:experiments:frl-results}} We estimate the LST and DST through \methodName\ and the trade-offs from the other baselines across various settings and datasets. \cref{fig:results:celeba:eo,fig:results:celeba:eoo,fig:results:celeba:dpv} show the trade-offs on the CelebA dataset for EOD, EOOD, and DPV, respectively. Similarly, \cref{fig:results:folktable:eo,fig:results:folktable:eoo,fig:results:folktable:dpv} show the trade-offs on the Folktable dataset for EOD, EOOD, and DPV, respectively. In the plots, the solid lines represent the mean, and the light shadows represent the variance of the accuracy for a given fairness value. On FairFace, observe that trade-offs do not exist since, on this task, it is possible to mitigate unfairness without sacrificing accuracy. Hence, we present the results of FRL methods and the estimated LST and DST in \cref{tab:results:fairface}. 

\newcommand{\cdst}[0]{\cellcolor{fillcolor6}}
\newcommand{\clst}[0]{\cellcolor{fillcolor}}
\begin{table}[!ht]
  \centering
    \resizebox{\columnwidth}{!}{%
    \begin{tabular}{clcccc}
    \toprule
    & Method & Accuracy ($\uparrow$) & Unfairness ($\downarrow$) & $\text{Dist}_{DST}$ ($\downarrow$) & $\text{Dist}_{LST}$ ($\downarrow$)  \\
    \midrule
    \multirow{5}{*}{\begin{sideways} EOD \end{sideways}}
    & ARL \cite{xie2017controllable}                      &  93.39  &  1.34  & 0.448  & 0.559\\ 
    & FairHSIC \cite{quadrianto2019discovering}              &  91.02  &  1.33  & 0.445  & 0.557\\ 
    & OptNet-ARL \cite{sadeghi2021adversarial}    &  92.94  &  1.70  & 0.598  & 0.709\\ 
    & \cdst \methodName-DST                                   &  \cdst 96.17  &  \cdst 0.263  & \cdst -    & \cdst 0.133\\  
    & \clst \methodName-LST                                   &  \clst 100.0  &   \clst 0.0  & \clst -  & \clst - \\ 
    \midrule
    \multirow{5}{*}{\begin{sideways} EOOD \end{sideways}}
    & ARL \cite{xie2017controllable}                      &  91.60  &  3.04  & 0.447  & 0.71\\ 
    & FairHSIC \cite{quadrianto2019discovering}              &  93.43  &  2.13  & 0.236  & 0.498\\ 
    & OptNet-ARL \cite{sadeghi2021adversarial}    &  93.39  &  2.34  & 0.284  & 0.546\\ 
    & \cdst \methodName-DST                             &  \cdst 97.93  &  \cdst 1.126  & \cdst -  & \cdst 0.262 \\ 
    & \clst \methodName-LST                             &  \clst 100.0  &  \clst 0.0   & \clst -  & \clst - \\ 
    \midrule
    \multirow{5}{*}{\begin{sideways} DPV \end{sideways}}
    & ARL \cite{xie2017controllable}                      & 92.49  &  6.09   &  0.350  & 0.351\\ 
    & FairHSIC \cite{quadrianto2019discovering}              & 91.41  &  5.91   &  0.329  & 0.332\\ 
    & OptNet-ARL \cite{sadeghi2021adversarial}    & 93.33  &  5.80   &  0.316 & 0.317\\ 
    & \cdst\methodName-DST             & \cdst 94.39  & \cdst 3.082   & \cdst - & \cdst 0.04 \\ 
    & \clst \methodName-LST            &  \clst 100.0 &   \clst 3.10   &  \clst - & \clst - \\ 
    \bottomrule
    \end{tabular}}%
  \caption{Evaluation of FRL methods on FairFace based on the distance to DST and LST estimated by \methodName. Color corresponds to the DST and LST trade-offs.\label{tab:results:fairface}}
  \vspace{-1em}
\end{table}%

\vspace{3pt}
\noindent\textbf{Observations:\label{secf:experiments:frl-observations}} Although K-$\mathcal T_{\text{Opt}}$, OptNet-ARL and FairHSIC can achieve near-optimal accuracy in most cases, they are unable to span the whole range of fairness values. ARL is the most unstable but can span the whole range of fairness values. K-$\mathcal T_{\text{Opt}}$ is the most stable method due to its closed-form solver, but is unable to span the whole range of fairness values (\cref{fig:results:celeba:dpv,fig:results:folktable:dpv}).

The gap between LST and DST demonstrates the information gap in $X$ for predicting $Y$. From \cref{fig:results:eo-eoo-dpv}, we observe that the gap is $\sim$20\% of accuracy in low fairness regions and $\sim$40\% in high fairness regions for EOD and EOOD. For DPV, the trend reverses with a gap of $\sim$20\% for low fairness and gradually decreases with increasing unfairness.

Observe that the LST in \cref{fig:results:celeba:eo,fig:results:celeba:eoo} and \cref{fig:results:folktable:eo,fig:results:folktable:eoo} for EOD and EOOD is almost flat at 100\% accuracy. This observation, however, is unsurprising since EOD and EOOD both condition on the label $Y$, and thus an ideal classifier with 100\% accuracy (i.e., $\hat{Y}=Y$) will have zero EOD and EOOD. And, in LST, the Oracle classifier is 100\% accurate since it has access to $Y$ and $S$. So, the LST has sufficient information to minimize EOD and EOOD without sacrificing utility. The same, however, does not hold for DPV since it does not consider the target labels in its definition. Based on the above discussion, we deduce that EOD and EOOD are more pragmatic fairness metrics than DPV since they do not force the model to sacrifice predictive accuracy to ensure fairness. Thus, both offer a more balanced and practical approach to measuring fairness. Our empirical results provides independent confirmation of the same observations in \cite{chouldechova2017fair, hardt2016equality}.

The comparison of FRL methods in \cref{tab:results:fairface} based on $\text{Dist}_{LST}$ suggests that when models are optimized for EOD, ARL and FairHSIC find solutions closer to LST and DST than OptNet-ARL. When models are optimized for EOOD, FairHSIC finds the closest point to the LST and DST. OptNet-ARL performs slightly better than the other FRL methods when optimized for reducing DPV. 

\subsection{Evaluating Zero-Shot CLIP Models\label{sec:experiments:results-vlm}}

\vspace{3pt}
\noindent\textbf{Zero-Shot Models:\label{secf:experiments:zeroshot-baselines}} To study the fairness of current zero-shot models, we consider more than 100 pre-trained models from OpenCLIP~\cite{ilharco_gabriel_2021_5143773} and evaluate them on CelebA and FairFace for the same target and sensitive labels as before.  

\vspace{3pt}
\noindent\textbf{Results:\label{secf:experiments:zeroshot-results}} \cref{fig:results:celeba-zeroshot,fig:results:fairface-zeroshot} show results on CelebA and FairFace, respectively, for three group fairness definitions. Each point represents the result of one zero-shot CLIP model, with the color denoting the model's pre-training dataset. Plots also include DST and LST for comparison.

\begin{figure*}[!ht]
    \centering
    \begin{subfigure}[c]{0.49\linewidth}
    \centering
    \includegraphics[width=\linewidth]{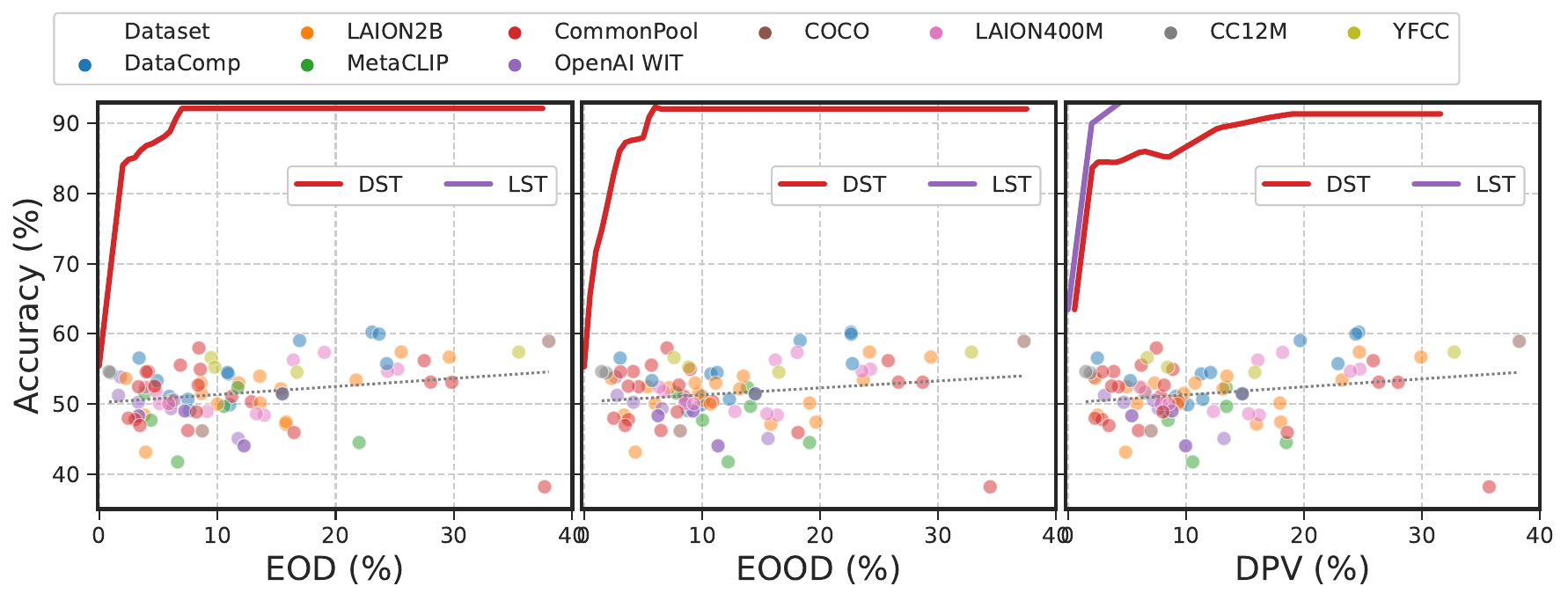}
    \caption{\label{fig:results:celeba-zeroshot}}  
    \end{subfigure}
    \begin{subfigure}[c]{0.49\linewidth}
        \centering
    \includegraphics[width=\linewidth]{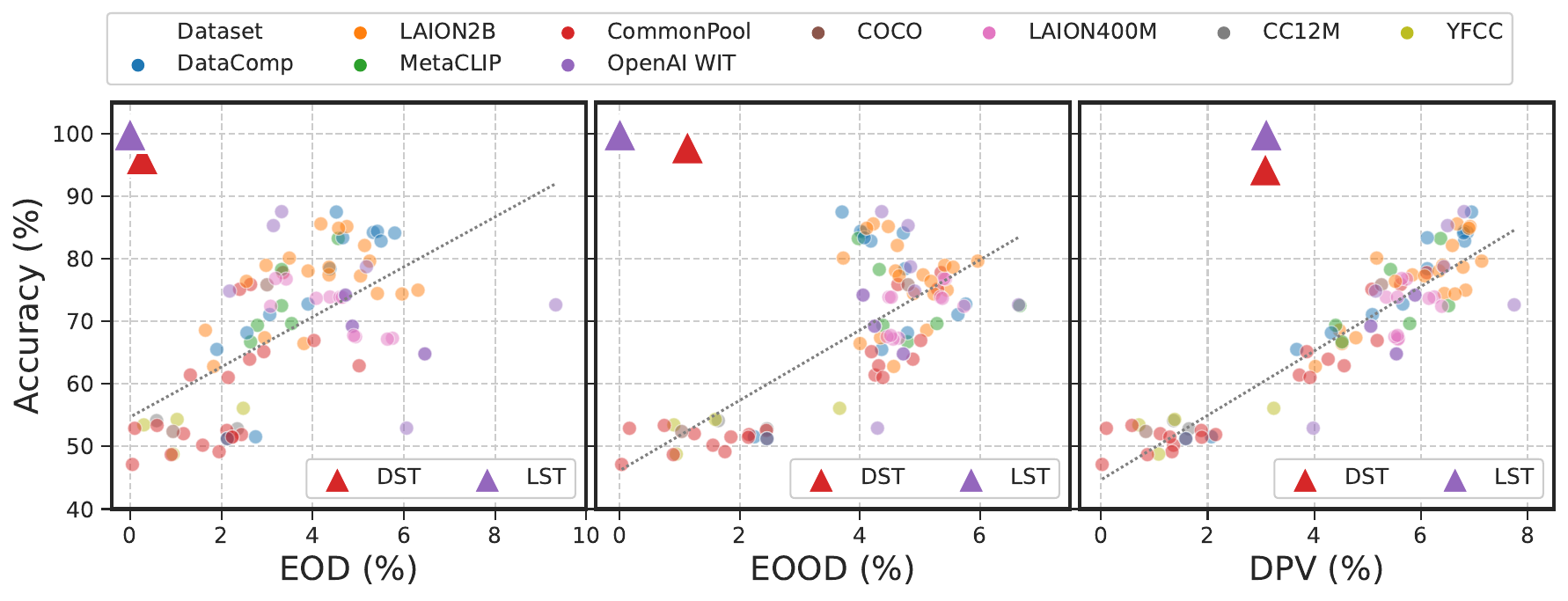}
    \caption{\label{fig:results:fairface-zeroshot}}   
    \end{subfigure}
    \begin{subfigure}[c]{0.49\linewidth}
    \centering
    \includegraphics[width=\linewidth]{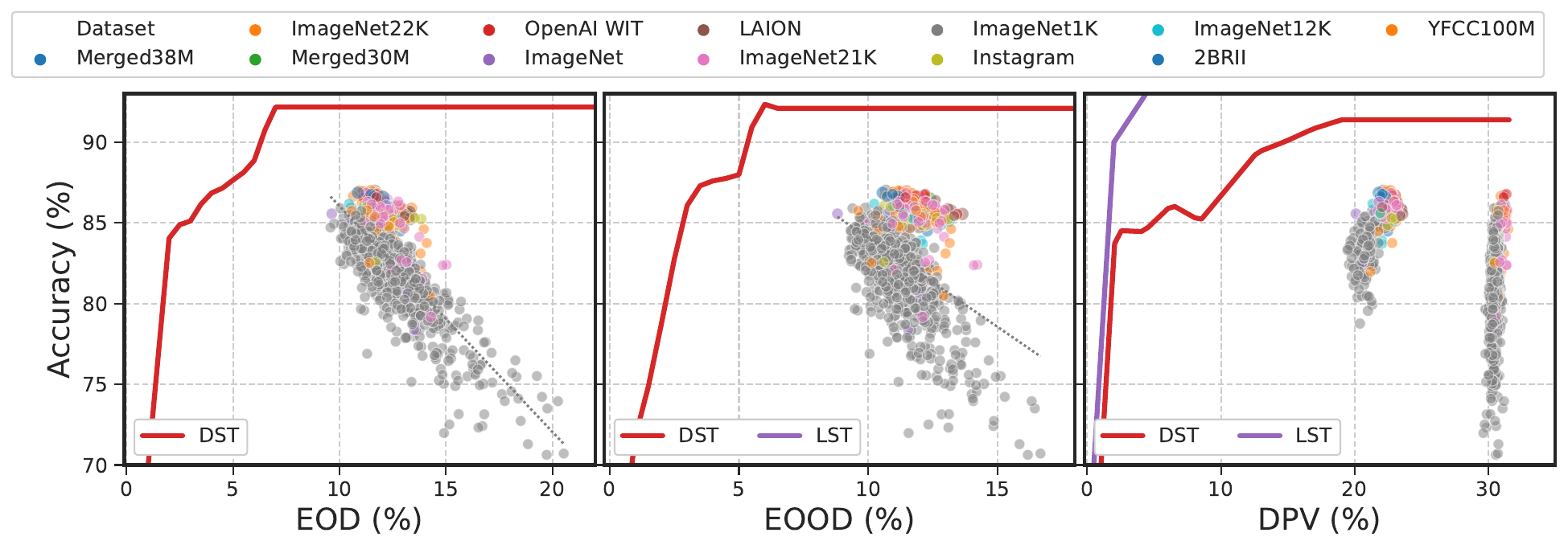}
    \caption{\label{fig:results:celeba-supervised}}    
    \end{subfigure}
    \begin{subfigure}[c]{0.49\linewidth}
    \centering
    \includegraphics[width=\linewidth]{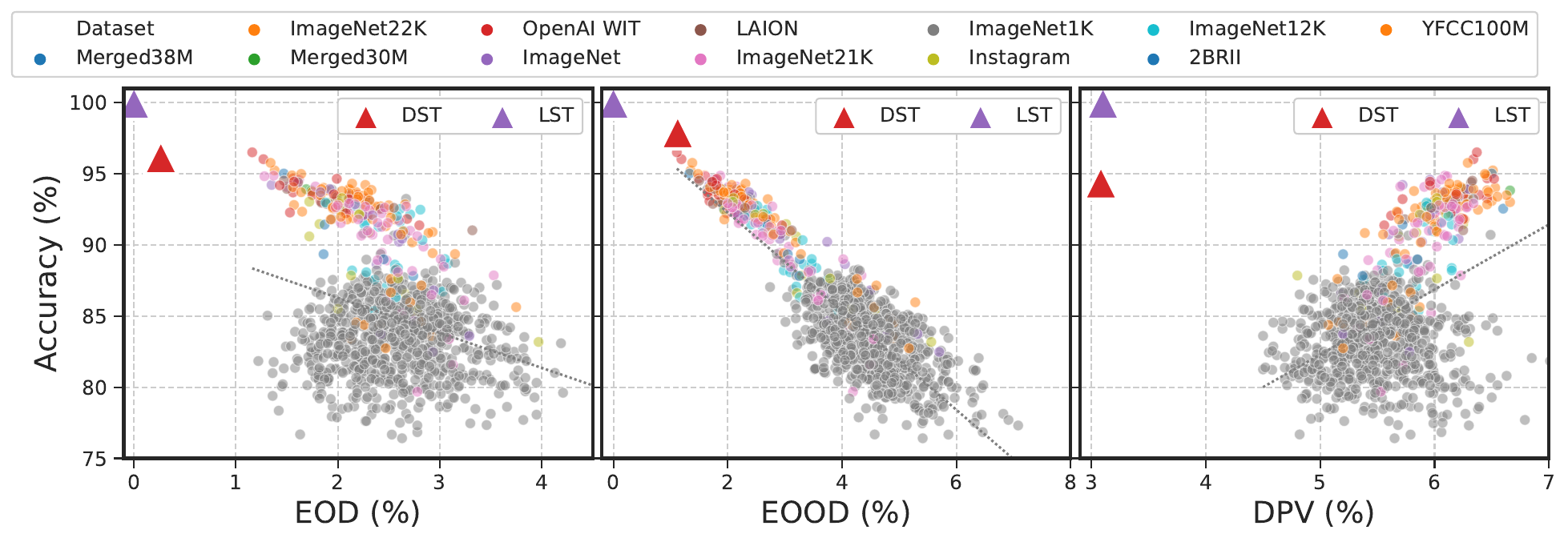}
    \caption{\label{fig:results:fairface-supervised}}   
    \end{subfigure}
    \vspace{-0.75em}
    \caption{\textbf{Evalauting Pre-Trained Zero-Shot and Supervised Models:} Accuracy-fairness evaluation of more than 100 pre-trained zero-shot models on CelebA (a) and  FairFace (b), and over 900 pre-trained image representations on CelebA (c), and  FairFace (d).\label{fig:results:zeroshot-supervised}}
    \vspace{-1em}
\end{figure*}

\vspace{3pt}
\noindent\textbf{Observations:\label{secf:experiments:zeroshot-observations}} 
From \cref{fig:results:celeba-zeroshot}, we observe that zero-shot models perform poorly, in terms of accuracy, on the CelebA task. We hypothesize that it is so since the target task (predicting \emph{high cheekbones}) is uncommon, even in the large datasets the models have been trained on. From a fairness perspective, we observe that models trained on CommonPool~\cite{gadre2023datacomp} (red dots) are more likely to be fair, while models trained on DataComp~\cite{gadre2023datacomp} have marginally better accuracy over the other models. Finally, the CLIP models are very far from the DST across the board.

On FairFace (\cref{fig:results:fairface-zeroshot}), since the target task (sex) is abundantly represented in all large-scale pre-training datasets, we observe that the CLIP models exhibit high levels of accuracy. Two CLIP models pre-trained on OpenAI WIT~\cite{radford2021learning} are the closest to DST and LST w.r.t. EOD. Similar to our observations on CelebA, models pre-trained on the CommonPool dataset are more likely to be fair but at the cost of accuracy. In contrast, models pre-trained on DataComp are more unfair but have greater accuracy. Finally, there is a clear positive correlation between accuracy and unfairness across all fairness metrics on FairFace (\cref{fig:results:fairface-zeroshot}).

\subsection{Evaluating Supervised Representations \label{sec:experiments:results-supervised}}
\noindent\textbf{Supervised Baselines:\label{secf:experiments:supervised-baselines}} To study the fairness of image representations from models pre-trained in a supervised fashion, we consider more than 900 models from Pytorch Image Models \cite{rw2019timm} and evaluate them on CelebA and FairFace for the same target and sensitive labels as before.

\vspace{3pt}
\noindent\textbf{Results:\label{sec:experiments:supervised-results}} \cref{fig:results:celeba-supervised,fig:results:fairface-supervised} show results on CelebA and FairFace datasets, respectively. Each point represents the result of one supervised model, with color denoting the model's pre-training dataset. Plots also include DST and LST for comparison, but some plots were magnified for better resolution. So, the LST may only be partially visible. 

\vspace{3pt}
\noindent\textbf{Observations:\label{sec:experiments:supervised-observations}} From the CelebA results in \cref{fig:results:celeba-supervised}, we observe a high positive correlation between the accuracy and EOD and EOOD. Thus, models with better accuracy are also more fair. And, in contrast to the zero-shot models, even though \emph{high cheekbones} is a rare label, the representations have sufficient information to accurately detect it with a logistic regression classifier. Specifically, models pre-trained on ImageNet22K \cite{ridnik2021imagenet} (orange dots), and OpenAI WIT \cite{radford2021learning} have the best accuracy and reach the DST trade-off in high unfairness regions. However, with more than 11\% EOD and EOOD and more than 20\% DPV, they have significant levels of bias between the two sexes.

Results on FairFace in \cref{fig:results:fairface-supervised} also reiterate that models trained on OpenAI WIT and ImageNet22K are more fair and more accurate than other datasets. We also observe that the models are generally more fair on FairFace than on CelebA. A similar positive correlation exists between accuracy and fairness for EOD and EOOD. We also make an interesting observation from \cref{fig:results:fairface-supervised}. Some models surpass the DST and enter the \emph{``Possible with Extra Data"} region of the utility-fairness plane illustrated in \cref{fig:teaser-ideal}. Recall that DST estimates the upper bound of the \emph{possible} region for models trained with the same data as DST. The DST can be surpassed to enter the \emph{``Possible with Extra Data"} region if the model is trained on additional data beyond what is used for estimating DST. From \cref{fig:results:fairface-supervised} (middle), we observe that a few models trained on OpenAI, LAION~\cite{schuhmann2022laion}, ImageNet22K, ImageNet21K, and ImageNet12K can surpass the DST, plausibly since these datasets contain sufficient---both quality and quantity---samples from the distribution of the target attribute (\emph{sex}).

%% file: 07-conclusion.tex
\section{Concluding Remarks\label{sec:conclusion}}
As image classification systems are widely deployed in high-stakes applications, ensuring that their predictions do not exhibit demographic bias is paramount for gaining user trust. While it is desirable to mitigate bias without sacrificing accuracy, this is not always possible. This paper studied such inherent trade-offs between utility and fairness. First, we identified two types of trade-offs called \emph{Data-Space} and \emph{Label-Space} trade-offs corresponding to those achievable with and without data restrictions. But unlike prior theoretical studies on utility-fairness trade-offs, next, we focused on developing algorithmic tools for quantifying the trade-offs from data and proposed \methodName{}. As an illustration of its practical utility, we estimated the trade-offs on several image classification tasks, facilitating a large-scale evaluation of over 100 zero-shot and 900 supervised pre-trained models. The results revealed that, out of the box, pre-trained models are far from the best achievable limits of accuracy and fairness. Furthermore, we identified that, in some cases, larger datasets can improve accuracy and fairness and surpass the solutions represented by the DST.

\methodName{} was designed as a composition of a neural network with the last layer optimized to global optimality through a closed-form solver for a given feature representation. This design allowed it to estimate the two trade-offs reliably. However, \methodName{} does not provide convergence or optimality guarantees. Therefore, the estimated DST and LST are likely to be suboptimal. Nonetheless, they can serve as a valuable tool for understanding the nature of the problem (e.g., Does there exist a trade-off? or What are the \emph{possible}, \emph{impossible} and \emph{possible with extra data} regions in the utility-fairness plane?) and how far a given fair learning algorithm is from the achievable limits. Furthermore, depending on which region of the utility-fairness plane a solution is and how far it is from the DST and LST reveals whether to focus on better optimization or better data.

\noindent\textbf{Acknowledgements:} This work was supported by the National Science Foundation (award \#2147116).

%% file: 08-appendix-arxiv.tex
\onecolumn
\section{Appendix\label{sec:appendix}}
In our main paper, we introduced two types of utility-fairness trade-offs and proposed a method, \methodName, to estimate them. Here, we provide some additional analysis to support our main results. The supplementary material is structured as follows:

\begin{enumerate}
    \item Representation Disentanglement in (\S\ref{sec:app:disentanglement})
    \item Training Process of \methodName\ in (\S\ref{sec:app:training})
    \item Implementation Details in (\S\ref{sec:app:implementation-details})
    \item Evaluation Metrics in (\S\ref{sec:app:metrics})
    \item Weighted Normalized Euclidean Distance in (\S\ref{sec:app:norm-euc})
    \item Proofs of closed-form solutions for different notions of fairness (\S\ref{sec:app:proofs})
\end{enumerate}

\subsection{Disentanglement of the Representation \label{sec:app:disentanglement}}
A common objective of learned representations is compactness~\cite{bengio2013representation} to avoid learning representations with redundant information where different dimensions are highly correlated. Therefore, going beyond the assumption that each component of $\bm f(\cdot )$ (i.e., $f_j(\cdot)$) belongs to a universal RKHS $\mathcal H_X$, we impose additional constraints on the representation. Specifically, we constrain the search space of the encoder $\bm f(\cdot)$ to learn a disentangled representation~\cite{bengio2013representation} as %
\begin{eqnarray}\label{eq:A}
\mathcal A_r:=\Big\{\left(f_1,\cdots, f_r\right) \,\big|\, f_i, f_j\in \mathcal H_{\tilde{X}},\, \cov \left(f_i({\tilde{X}}), f_j({\tilde{X}}) \right)
+\,\gamma\, \langle f_i, f_j\rangle_{\mathcal H_{\tilde{X}}}=\delta_{i,j}\Big\},
\end{eqnarray}
where the $\cov \big(f_i(X), f_j(X) \big)$ part enforces the covariance of $Z$ to be an identity matrix. This kind of disentanglement is used in PCA and encourages the variance of each entry of $Z$ to be bounded and different entries of $Z$ are uncorrelated to each other. The regularization part, $\gamma\,\langle f_i, f_j\rangle_{\mathcal H_X}$ encourages the encoder components to be as orthogonal as possible to each other and to be of the unit norm, and aids with numerical stability during empirical estimation~\cite{fukumizu2007statistical}.

\subsection{Training Process of \methodName \label{sec:app:training}}

\input{algorithm}

\input{figs/proposed-method}

\cref{fig:model-overview} shows an overview of the training process of \methodName\ which includes two phases. In the first phase, the features of the training samples are extracted and used to find a closed-form solution for the encoder to maximize the objective function in (4) while the parameters of the feature extractor ($\bm \Theta_{FE}$) are frozen. In the second phase, the feature extractor is trained by updating its parameters using SGD in order to maximize (4) while the encoder is frozen. These two phases are repeated until convergence. These details are also mentioned in \cref{alg:training}.

\subsection{Implementation Details\label{sec:app:implementation-details}}
In training all of the methods, we pick different values of the fairness control parameter ($\lambda$) between zero and one to obtain the trade-offs. Moreover, each experiment is run for 5 different random seeds. For datasets that contain image data, we used the first two blocks of ResNet18 \cite{he2016deep} and put a fully connected layer with 2048 neurons as the last layer of the feature extractor. We used an embedding layer for the dataset with tabular data to map the raw data into an embedding space. 
A 3-layer MLP is used as the target classifier network for all datasets and models. 
For both FolkTables and CelebA datasets, the number of dimensions of RFF is set to 1000. 
In the training phase, the cosine annealing scheduler \cite{loshchilov2016sgdr} is used for scheduling the learning rate. The dimension of representations ($r$) is chosen $c - 1$ where $c$ is the number of target attribute's classes. To improve training stability, we normalize the feature extractor's output $\tilde{X}$. 
These implementation details are summarized in \cref{tab:implementation-details}.
\begin{table}[ht]
    \centering
    \begin{tabular}{lccc}
    \toprule
    Dataset    & RFF Dim.  & $r$ & Training samples \\
    \midrule
    CelebA     &  1000 & 1 & 182,637\\
    FairFace   &  1000 & 1 & 86,744\\
    FolkTables &  1000 & 1 & 75,745 \\
    \bottomrule
    \end{tabular}
    \caption{Implementation details of \methodName\ for each dataset.}
    \label{tab:implementation-details}
\end{table}

\subsection{Evaluation Metrics\label{sec:app:metrics}}
For measuring the utility of the target prediction, we use the accuracy of the classification task. Furthermore, we use equality of odds difference
(EOOD) or equal opportunity difference (EOD). EOOD\cite{hardt2016equality} is defined as
\begin{equation}
\label{eq:eood}
\footnotesize{
\begin{aligned}
    \text{EOOD} := \left| P\left( \hat{Y} = 1 | S = 0, Y = y \right) - P\left( \hat{Y} = 1 | S = 1, Y = y \right) \right|,
\end{aligned}}
\end{equation}
where $y \in \{0, 1, \cdots, |Y|-1\}$, $\hat{Y}$ is the predicted label, and $S$ is the sensitive attribute. According to this criterion, the model should exhibit similar prediction error rates for different groups, irrespective of their sensitive attributes.
EOD\cite{hardt2016equality} can also be defined as
\begin{equation}
\label{eq:eod}
\footnotesize{
\begin{aligned}
    \text{EOD} := \left| P\left( \hat{Y} = 1 | S = 0, Y = 1 \right) - P\left( \hat{Y} = 1 | S = 1, Y = 1 \right) \right|
\end{aligned}}
\end{equation}
This is a relaxation of EOOD for the case of binary target tasks. This metric indicates that the model should provide equal opportunities for positive predictions for individuals with the same true outcome, regardless of their sensitive attributes.

\subsection{Weighted Normalized Euclidean Distance \label{sec:app:norm-euc}}
In the main paper, to compare methods based on their point-to-point distance to LST and DST, we use two \emph{weighted normalized Euclidean distance} defined as:
\begin{equation}
\label{eq:dist-lst}
\footnotesize{
\begin{aligned}
    \text{Dist}_\text{LST}(x^i) = \sqrt{w \cdot \left(\frac{\text{LST}_f - x^i_f}{\max_f} \right)^2 + (1 - w) \cdot \left(\frac{\text{LST}_\text{Acc} - x^i_\text{Acc}}{\max_\text{Acc} } \right)^2}
\end{aligned}}
\end{equation}

\begin{equation}
\label{eq:dist-dst}
\footnotesize{
\begin{aligned}
    \text{Dist}_\text{DST}(x^i) = \sqrt{w \cdot \left(\frac{\text{DST}_f - x^i_f}{\max_f} \right)^2 + (1 - w) \cdot \left(\frac{\text{DST}_\text{Acc} - x^i_\text{Acc}}{\max_\text{Acc} } \right)^2}
\end{aligned}}
\end{equation}
where $f \in \mathcal{F}$ is the fairness metric and $\mathcal{F} = \left\{\text{EOD}, \text{EOOD}, \text{DPV}\right\}$, $w$ is the control parameter that adjusts the weights of each term ---fairness distance and accuracy distance---in the overall distance. For calculating distances in Table 1 of the main paper, we choose $w = 0.5$ which means that the distances in the fairness axis and distances in the accuracy axis are equally important to us.

\subsection{Solutions for Different Notions of Fairness \label{sec:app:proofs}}

\subsubsection{Proof of Theorem 1 for EO}
\label{sec:eo-emp}
\begin{theorem1}
Let the Cholesky factorization of $\bm K_X$ be $\bm K_X=\bm L_X \bm L_X^T$,  where $\bm L_X\in \mathbb R^{n\times d}$ ($d\le n$) is a full column-rank matrix. Let $r\le d$, then a solution to (4) is 
\begin{eqnarray}
\bm f^{\text{opt}}(X) =
\bm \Theta^{\text{opt}}
\left[k_X(\bm x_1, X),\cdots, k_X(\bm x_n, X)\right]^T\nn
\end{eqnarray}
where $\bm \Theta^{\text{opt}}=\bm U^T \bm L_X^\dagger$
and the columns of $\bm U$
are eigenvectors corresponding to the $r$ largest eigenvalues of the following generalized eigenvalue problem.
\begin{eqnarray}
\bm \left( (1-\lambda) \frac{1}{n^2} \bm L_{X}^T \bm H \bm K_{Y} \bm H \bm L_{X} -\lambda \frac{1}{n_0^2} \bm L_{X}[Y=y_0]^T \bm H \bm K_{S}[Y=y_0] \bm H \bm L_{X}[Y=y_0] \right) \bm u = \tau \left(\frac{1}{n}\,\bm L^T_X \bm H \bm L_X + \gamma \bm I\right) \bm u.
\end{eqnarray}
\end{theorem1}
\begin{proof}
Consider the Cholesky factorization, $\bm K_X=\bm L_X \bm L_X^T$ where $\bm L_X$ is a full column-rank matrix. Using the representer theorem, the disentanglement property in ~\eqref{eq:A} can be expressed as
\begin{eqnarray}
&&\cov \left(f_i(X),\,f_j(X) \right) + \gamma\, \langle f_i, f_j \rangle_{\mathcal H_{X}}\nn\\
&=& \frac{1}{n}\sum_{k=1}^nf_i(\bm x_k) f_j(\bm x_k) -\frac{1}{n^2}\sum_{k=1}^n f_i(\bm x_k)\sum_{m=1}^n f_j(\bm x_m) + \gamma\, \langle f_i, f_j \rangle_{\mathcal H_{X}} \nn\\
&=&\frac{1}{n}\sum_{k=1}^n \sum_{t=1}^n \bm K_{X} (\bm x_k, \bm x_t)\theta_{i t}
\sum_{m=1}^n \bm K_{X} (\bm x_k, \bm x_m)\theta_{j m}-\frac{1}{n^2}\bm \theta_i^T \bm K_{X} \bm 1_n \bm 1_n^T \bm K_{X}\bm \theta_j+\gamma\, \langle f_i, f_j \rangle_{\mathcal H_{X}}\nn\\
&=& \frac{1}{n} \left(\bm K_{X} \bm \theta_i\right)^T
\left(\bm K_X \bm \theta_j\right)-\frac{1}{n^2}\bm \theta_i^T \bm K_{X} \bm 1_n \bm 1_n^T \bm K_{X}\bm \theta_j+\gamma\,
\left\langle \sum_{k=1}^n \theta_{ik}k_{X}(\cdot, \bm x_k), \sum_{t=1}^n \theta_{it}k_{X}(\cdot, \bm x_t)\right\rangle_{\mathcal H_{X}} \nn\\
&=& \frac{1}{n} \bm \theta_i^T \bm K_{X}\bm H \bm K_{X} \bm \theta_j
+ \gamma\,\bm \theta^T_i \bm K_{X} \bm \theta_j\nn\\
&=& \frac{1}{n} \bm \theta_i^T \bm L_{X} \left(\bm L^T_{X}\bm H \bm L_{X} + n\gamma\,\bm I\right) \bm L^T_{X} \bm \theta_j\nn\\
&=&\delta_{i,j}.\nn
\end{eqnarray}
As a result, $\bm f\in \mathcal A_r$ is equivalent to 
\begin{eqnarray}
\bm \Theta \bm L_{X} \underbrace{\Big( \frac{1}{n}\bm L^T_{X}\bm H \bm L_{X} + \gamma\bm I\Big)}_{:=\bm C} \bm L^T_{X} \bm \Theta^T= \bm I_r\nn,
\end{eqnarray}
where $\bm \Theta:=\big[ \bm \theta_1,\cdots, \bm \theta_r\big]^T\in \mathbb R^{r\times n}$.

Let $\bm V = \bm L_{X}^T\bm \Theta ^T $ and consider the optimization problem in~(13): 
 \begin{eqnarray}\label{eq:trace-eo}
 &&\sup_{\bm f \in \mathcal A_r} \left\{(1-\lambda)\,\text{Dep}^{\text{emp}}(\bm f(X), Y) -
\lambda\, \text{Dep}^{\text{emp}}(\bm f(X), S | Y=1)\right\}\nn\\
 &=&\sup_{\bm f \in \mathcal A_r} \left\{(1-\lambda)\frac{1}{n^2}\left\|\bm \Theta \bm K_{X} \bm H \bm L_{Y} \right\|^2_F
 -\lambda\, \frac{1}{n_0^2}\left\|\bm \Theta \bm K_{X}[Y=y_0] \bm H \bm L_{S_0} \right\|^2_F\nn\right\}\nn\\
  &=&\sup_{\bm f \in \mathcal A_r } \left\{(1-\lambda)\frac{1}{n^2}\,\text{Tr}\left\{\bm \Theta \bm K_{X} \bm H \bm K_{Y} \bm H \bm K_{X}\bm \Theta^T\right\}
 -\lambda \,\frac{1}{n_0^2}\Tr{\bm \Theta \bm K_{X}[Y=y_0] \bm H \bm K_{S_0} \bm H \bm K_{X}[Y=y_0]^T \bm \Theta^T}\right\}\nn\\
  &=&\max_{\bm V^T \bm C \bm V = \bm I_r}  \text{Tr} \left\{\bm\Theta \bm L_{X}  \bm B \bm L_{X}^T \bm \Theta^T\right\}\nn\\
&=&\max_{\bm V^T \bm C \bm V = \bm I_r}  \text{Tr} \left\{ \bm V^T  \bm B \bm V \right\}
 \end{eqnarray}
where the second step is due to (3) and
\begin{eqnarray}
\bm B&:=& \left( (1-\lambda) \frac{1}{n^2} \bm L_{X}^T \bm H \bm K_{Y} \bm H \bm L_{X} -\lambda \frac{1}{n_0^2} \bm L_{X}[Y=y_0]^T \bm H \bm K_{S}[Y=y_0] \bm H \bm L_{X}[Y=y_0] \right)\nn
\end{eqnarray}
It is shown in~\cite{kokiopoulou2011trace} that an\footnote{Optimal $\bm V$ is not unique.} optimizer of~(\ref{eq:trace-eo}) is any matrix $\bm U$ whose columns are eigenvectors corresponding to $r$ largest eigenvalues of generalized problem
\begin{eqnarray}\label{eq:eig-gen-proof-eo}
\bm B \bm u = \tau \,\bm C \bm u 
\end{eqnarray}
and the maximum value is the summation of $r$ largest eigenvalues. Once $\bm U$ is determined, then, any $\bm \Theta$ in which $\bm L_{X}^T\bm \Theta^T = \bm U$ is optimal $\bm \Theta$ (denoted by $\bm \Theta^{\text{opt}}$).
Note that $\bm \Theta^{\text{opt}}$ is not unique and has a general form of
\begin{eqnarray}
\bm \Theta^T = \left( \bm L_{X}^T\right)^\dagger \bm U + \bm \Lambda_0, \quad  \mathcal R(\bm \Lambda_0)\subseteq \mathcal N \left( \bm L^T_{X}\right).\nn
\end{eqnarray}
However, setting $\bm \Lambda_0$ to zero would lead to minimum norm for  $\bm \Theta$. Therefore, we opt $\bm \Theta^{\text{opt}}=\bm U^T \bm L_{X}^\dagger$.
\end{proof}

\subsubsection{Proof of Theorem 1 for EOO}
\label{sec:eoo-emp}
\begin{theorem1}
Let the Cholesky factorization of $\bm K_X$ be $\bm K_X=\bm L_X \bm L_X^T$,  where $\bm L_X\in \mathbb R^{n\times d}$ ($d\le n$) is a full column-rank matrix. Let $r\le d$, then a solution to (4) is 
\begin{eqnarray}
\bm f^{\text{opt}}(X) =
\bm \Theta^{\text{opt}}
\left[k_X(\bm x_1, X),\cdots, k_X(\bm x_n, X)\right]^T\nn
\end{eqnarray}
where $\bm \Theta^{\text{opt}}=\bm U^T \bm L_X^\dagger$
and the columns of $\bm U$
are eigenvectors corresponding to the $r$ largest eigenvalues of the following generalized eigenvalue problem.
\begin{eqnarray}
\bm \left( (1-\lambda) \frac{1}{n^2} \bm L_{X}^T \bm H \bm K_{Y} \bm H \bm L_{X} -\lambda \sum_{y = 0}^{c_y - 1} \frac{1}{n_y^2} \bm L_{X}[Y=y]^T \bm H \bm K_{S}[Y=y] \bm H \bm L_{X}[Y=y] \right) \bm u = \tau \left(\frac{1}{n}\,\bm L^T_X \bm H \bm L_X + \gamma \bm I\right) \bm u.
\end{eqnarray}
\end{theorem1}
\begin{proof}
Consider the Cholesky factorization, $\bm K_X=\bm L_X \bm L_X^T$ where $\bm L_X$ is a full column-rank matrix. Using the representer theorem, the disentanglement property in ~\eqref{eq:A} can be expressed as
\begin{eqnarray}
&&\cov \left(f_i(X),\,f_j(X) \right) + \gamma\, \langle f_i, f_j \rangle_{\mathcal H_{X}}\nn\\
&=& \frac{1}{n}\sum_{k=1}^nf_i(\bm x_k) f_j(\bm x_k) -\frac{1}{n^2}\sum_{k=1}^n f_i(\bm x_k)\sum_{m=1}^n f_j(\bm x_m) + \gamma\, \langle f_i, f_j \rangle_{\mathcal H_{X}} \nn\\
&=&\frac{1}{n}\sum_{k=1}^n \sum_{t=1}^n \bm K_{X} (\bm x_k, \bm x_t)\theta_{i t}
\sum_{m=1}^n \bm K_{X} (\bm x_k, \bm x_m)\theta_{j m}-\frac{1}{n^2}\bm \theta_i^T \bm K_{X} \bm 1_n \bm 1_n^T \bm K_{X}\bm \theta_j+\gamma\, \langle f_i, f_j \rangle_{\mathcal H_{X}}\nn\\
&=& \frac{1}{n} \left(\bm K_{X} \bm \theta_i\right)^T
\left(\bm K_X \bm \theta_j\right)-\frac{1}{n^2}\bm \theta_i^T \bm K_{X} \bm 1_n \bm 1_n^T \bm K_{X}\bm \theta_j+\gamma\,
\left\langle \sum_{k=1}^n \theta_{ik}k_{X}(\cdot, \bm x_k), \sum_{t=1}^n \theta_{it}k_{X}(\cdot, \bm x_t)\right\rangle_{\mathcal H_{X}} \nn\\
&=& \frac{1}{n} \bm \theta_i^T \bm K_{X}\bm H \bm K_{X} \bm \theta_j
+ \gamma\,\bm \theta^T_i \bm K_{X} \bm \theta_j\nn\\
&=& \frac{1}{n} \bm \theta_i^T \bm L_{X} \left(\bm L^T_{X}\bm H \bm L_{X} + n\gamma\,\bm I\right) \bm L^T_{X} \bm \theta_j\nn\\
&=&\delta_{i,j}.\nn
\end{eqnarray}
As a result, $\bm f\in \mathcal A_r$ is equivalent to 
\begin{eqnarray}
\bm \Theta \bm L_{X} \underbrace{\Big( \frac{1}{n}\bm L^T_{X}\bm H \bm L_{X} + \gamma\bm I\Big)}_{:=\bm C} \bm L^T_{X} \bm \Theta^T= \bm I_r\nn,
\end{eqnarray}
where $\bm \Theta:=\big[ \bm \theta_1,\cdots, \bm \theta_r\big]^T\in \mathbb R^{r\times n}$.

Let $\bm V = \bm L_{X}^T\bm \Theta ^T $ and consider the optimization problem in~(13): 
 \begin{eqnarray}\label{eq:trace-eoo}
 &&\sup_{\bm f \in \mathcal A_r} \left\{(1-\lambda)\,\text{Dep}^{\text{emp}}(\bm f(X), Y) -
\lambda\, \sum_{y=0}^{c_y-1} \text{Dep}^{\text{emp}}(\bm f(X), S | Y=y)\right\}\nn\\
 &=&\sup_{\bm f \in \mathcal A_r} \left\{(1-\lambda)\frac{1}{n^2}\left\|\bm \Theta \bm K_{X} \bm H \bm L_{Y} \right\|^2_F
 -\lambda\, \sum_{y=0}^{c_y-1} \frac{1}{n_y^2}\left\|\bm \Theta \bm K_{X}[Y=y] \bm H \bm L_{S_y} \right\|^2_F\nn\right\}\nn\\
  &=&\sup_{\bm f \in \mathcal A_r } \left\{(1-\lambda)\frac{1}{n^2}\,\text{Tr}\left\{\bm \Theta \bm K_{X} \bm H \bm K_{Y} \bm H \bm K_{X}\bm \Theta^T\right\}
 -\lambda \, \sum_{y=0}^{c_y-1} \frac{1}{n_y^2}\Tr{\bm \Theta \bm K_{X}[Y=y] \bm H \bm K_{S_y} \bm H \bm K_{X}[Y=y]^T \bm \Theta^T}\right\}\nn\\
  &=&\max_{\bm V^T \bm C \bm V = \bm I_r}  \text{Tr} \left\{\bm\Theta \bm L_{X}  \bm B \bm L_{X}^T \bm \Theta^T\right\}\nn\\
&=&\max_{\bm V^T \bm C \bm V = \bm I_r}  \text{Tr} \left\{ \bm V^T  \bm B \bm V \right\}
 \end{eqnarray}
where the second step is due to (3) and
\begin{eqnarray}
\bm B&:=& \left( (1-\lambda) \frac{1}{n^2} \bm L_{X}^T \bm H \bm K_{Y} \bm H \bm L_{X} -\lambda \sum_{y=0}^{c_y-1} \frac{1}{n_y^2} \bm L_{X}[Y=y]^T \bm H \bm K_{S}[Y=y] \bm H \bm L_{X}[Y=y] \right)\nn
\end{eqnarray}
It is shown in~\cite{kokiopoulou2011trace} that an\footnote{Optimal $\bm V$ is not unique.} optimizer of~(\ref{eq:trace-eoo}) is any matrix $\bm U$ whose columns are eigenvectors corresponding to $r$ largest eigenvalues of generalized problem
\begin{eqnarray}\label{eq:eig-gen-proof-eoo}
\bm B \bm u = \tau \,\bm C \bm u 
\end{eqnarray}
and the maximum value is the summation of $r$ largest eigenvalues. Once $\bm U$ is determined, then, any $\bm \Theta$ in which $\bm L_{X}^T\bm \Theta^T = \bm U$ is optimal $\bm \Theta$ (denoted by $\bm \Theta^{\text{opt}}$).
Note that $\bm \Theta^{\text{opt}}$ is not unique and has a general form of
\begin{eqnarray}
\bm \Theta^T = \left( \bm L_{X}^T\right)^\dagger \bm U + \bm \Lambda_0, \quad  \mathcal R(\bm \Lambda_0)\subseteq \mathcal N \left( \bm L^T_{X}\right).\nn
\end{eqnarray}
However, setting $\bm \Lambda_0$ to zero would lead to minimum norm for  $\bm \Theta$. Therefore, we opt $\bm \Theta^{\text{opt}}=\bm U^T \bm L_{X}^\dagger$.
\end{proof}

%% file: algorithm.tex
\RestyleAlgo{ruled}
\SetKwComment{Comment}{/* }{ */}
\begin{algorithm*}[ht]
\caption{\methodName\ Training Process\label{alg:training}}

{\small \textbf{Input:} $\bm X \in \mathbb{R}^{n \times h \times w \times c}$, $\bm Y \in \mathbb{R}^{n \times |Y|}$, $\bm S \in \mathbb{R}^{n \times |S|}$, $m \in \mathbb{N}$} \\

{\small \textbf{Output:} $\bm f_{FE}$, $\bm f_{Enc}$} \\ 

{\small \textbf{Initialize:}} \\

{\small $i \gets 0$\;}
$ \bm f_{FE} \gets \text{Random Initialization};$

\Comment*[l]{Train Feature Extractor and Encoder}
\While{$i < m$}{
$ \bm f_{Enc} \gets \sup_{\bm f_{Enc} \in \mathcal A_r} \{J^{\text{emp}}(\bm f(X; \bm \Theta_{Enc}))\}$\Comment*[r]{solve (6)}

$ \bm f_{FE} \gets \text{SGD} \Big\{ \sup_{\bm f_{FE}} \{J^{\text{emp}}(\bm f(X; \bm \Theta_{FE}))\} \Big\}$ \Comment*[r]{solve (2)}

{\small $i \gets i + 1$}
}

$ \bm f_{CLF} \gets \text{SGD} \Big\{ \sup_{\bm f_{CLF}} \{J(\bm f(X; \bm \Theta_{CLF}), \bm Y)\} \Big\}$\Comment*[r]{Train Classifier}

\end{algorithm*}

%% file: figs/proposed-method.tex
\FPset\figlinewidth{1.0}

\begin{figure*}[ht]
  \centering
  \scalebox{0.85}{
  \begin{tikzpicture}[node distance=3cm, every node/.style={font=\sffamily}]
    \node[anchor=center](input){\includegraphics[width=2.5cm]{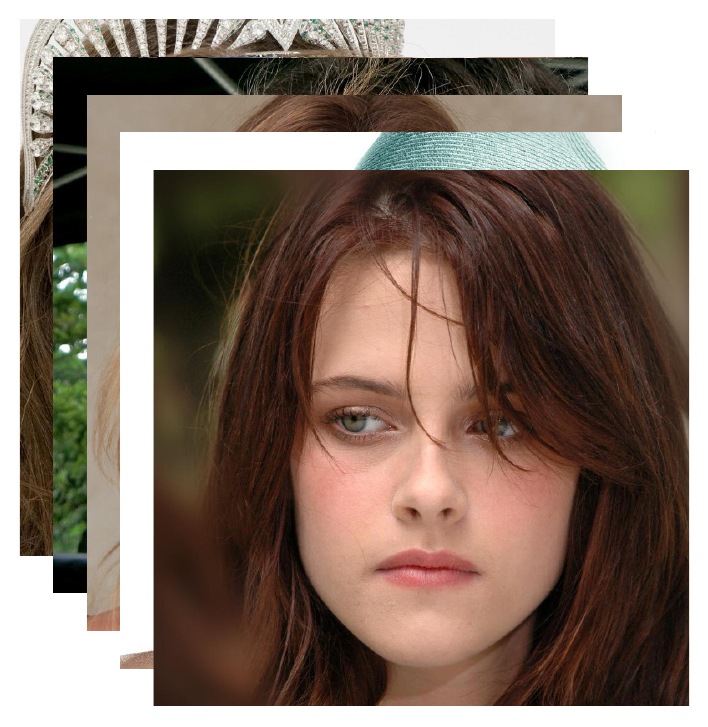}};
    \node[above] at ([yshift=0.5em, xshift=0em]input.north) {\textbf{Phase 1}};
    
    \node[draw=black, fill=fillcolor2, line width=\figlinewidth pt, rectangle, rounded corners, align=center, anchor=center, right of=input, minimum height=3cm, minimum width=1cm] (fe) {Feature\\Extractor\\$\bm \Theta_{FE}$};
    \node[draw=black, dash pattern=on 10pt off 5pt, fill=fillcolor5, line width=\figlinewidth pt, rectangle, rounded corners, align=center, anchor=center, xshift=4em, right of=fe, minimum height=3.0cm, minimum width=5.5cm] (box) {};
    \node[draw=black, fill=fillcolor3, line width=\figlinewidth pt, rectangle, rounded corners, align=center, anchor=west, minimum height=2.5cm, minimum width=0.5cm] (x_tilde) at ([yshift=0em, xshift=1.5em]box.west) {$\bm \tilde{\bm X}$};
    \node[draw=black, fill=fillcolor, line width=\figlinewidth pt, rectangle, rounded corners, align=center, anchor=center, minimum height=1cm, minimum width=0.5cm] (cfs) at ([yshift=0em, xshift=2em]box.center) {Closed-Form Solver};

    \node[anchor=center, below of=input, yshift=-3em](input2){\includegraphics[width=2.5cm]{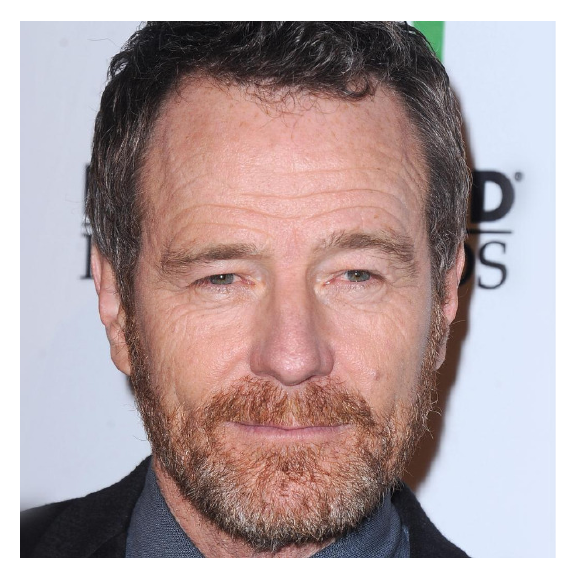}};   
    
    \node[above] at ([yshift=0.5em, xshift=0em]input2.north) {\textbf{Phase 2}};
    
    \node[draw=black, fill=fillcolor, line width=\figlinewidth pt, rectangle, rounded corners, align=center, anchor=center, right of=input2, minimum height=3cm, minimum width=1cm] (fe2) {Feature\\Extractor\\$\bm \Theta_{FE}$};
    
    \node[draw=black, dash pattern=on 10pt off 5pt, fill=fillcolor5, line width=\figlinewidth pt, rectangle, rounded corners, align=center, anchor=center, xshift=4em, right of=fe2, minimum height=3.0cm, minimum width=5.5cm] (box2) {};

    \node[draw=black, fill=fillcolor2, line width=\figlinewidth pt, rectangle, rounded corners, align=center, anchor=west, minimum height=1.3cm, minimum width=0.5cm] (enc) at ([yshift=0em, xshift=1em]box2.west) {Encoder\\$\bm \Theta_{Enc}$};

    \node[draw=black, fill=fillcolor3, line width=\figlinewidth pt, rectangle, rounded corners, align=center, anchor=west, minimum height=2.5cm, minimum width=0.5cm] (z) at ([yshift=0em, xshift=1.2em]enc.east) {$\bm Z$};
    
    \node[draw=black, fill=fillcolor7, line width=\figlinewidth pt, rectangle, rounded corners, align=center, anchor=east, minimum height=0.5cm, minimum width=0.5cm] (depy) at ([yshift=-1.5em, xshift=-0.9em]box2.north east) {$Dep(Z, Y)$};
    \node[draw=black, fill=fillcolor6, line width=\figlinewidth pt, rectangle, rounded corners, align=center, anchor=east, minimum height=0.5cm, minimum width=0.5cm] (deps) at ([yshift=1.5em, xshift=-0.9em]box2.south east) {$Dep(Z, S)$};
    
    \node[draw=black, fill=fillcolor8!60, line width=\figlinewidth pt, circle, align=center, anchor=center, minimum height=0.2cm, minimum width=0.1cm] (loss) at ([yshift=0em, xshift=4em]box2.east) {$\mathcal{L}$};

    \draw[line width=\figlinewidth pt, transform canvas={yshift=8mm},draw] [->] (input) -- (fe);
    \draw[line width=\figlinewidth pt, transform canvas={yshift=0mm},draw] [->] (input) -- (fe);
    \draw[line width=\figlinewidth pt, transform canvas={yshift=-8mm},draw] [->] (input) -- (fe);

    \draw[line width=\figlinewidth pt, transform canvas={yshift=8mm},draw] [->] (fe) -- (box.west);
    \draw[line width=\figlinewidth pt, transform canvas={yshift=0mm},draw] [->] (fe) -- (box.west);
    \draw[line width=\figlinewidth pt, transform canvas={yshift=-8mm},draw] [->] (fe) -- (box.west);
    
    \draw[line width=\figlinewidth pt, transform canvas={yshift=0mm},draw] [->] (x_tilde) -- (cfs);

    \draw[line width=\figlinewidth pt, transform canvas={yshift=0mm},draw, dashed] [->] (box) -- node[right] {$\bm \Theta_{Enc}$} (box2);

    \draw[line width=\figlinewidth pt, transform canvas={yshift=0mm},draw] [->] (input2) -- node[above]{$X$} (fe2);
    \draw[line width=\figlinewidth pt, transform canvas={yshift=0mm},draw] [->] (fe2) --node[above]{$\tilde{X}$} (box2.west);
    \draw[line width=\figlinewidth pt, transform canvas={yshift=0mm},draw] [->] (enc) -- (z);
    \draw[line width=\figlinewidth pt, transform canvas={yshift=0mm},draw] [->] (z) -- (depy.west);
    \draw[line width=\figlinewidth pt, transform canvas={yshift=0mm},draw] [->] (z) -- (deps.west);
    
    \draw[line width=\figlinewidth pt, transform canvas={yshift=0em},draw] [->] ([yshift=-1.5em]box2.north east) --  node[above, sloped] {$\scalemath{0.8}{\times -1}$}  (loss.north west);
    \draw[line width=\figlinewidth pt, transform canvas={yshift=0mm},draw] [->] ([yshift=1.5em]box2.south east) --  node[below, sloped] {$\scalemath{0.8}{\times \frac{\lambda}{1-\lambda}}$}  (loss.south west);

     \def\w {0.35}
     \node[draw=black, fill=fillcolor2, line width=\figlinewidth pt, rectangle, align=center, anchor=south, minimum height=\w cm, minimum width=\w cm] (leg1) at ([yshift=3em, xshift=2.5em]box.east) {};
     \node[right] (txt1) at ([yshift=0em, xshift=0em]leg1.east) {Frozen Params};
     
     \node[draw=black, fill=fillcolor, line width=\figlinewidth pt, rectangle, align=center, anchor=south, minimum height=\w cm, minimum width=\w cm] (leg2) at ([yshift=-1.5em, xshift=0.0em]leg1.south) {};
     \node[right] (txt2) at ([yshift=0em, xshift=0em]leg2.east) {Trainable Params};
     
     \node[draw=black, fill=fillcolor3, line width=\figlinewidth pt, rectangle, align=center, anchor=south, minimum height=\w cm, minimum width=\w cm] (leg3) at ([yshift=-1.5em, xshift=0.0em]leg2.south) {};
     \node[right] (txt3) at ([yshift=0em, xshift=0em]leg3.east) {Feature Space};
     
     \node[draw=black, fill=fillcolor8!60, line width=\figlinewidth pt, rectangle, align=center, anchor=south, minimum height=\w cm, minimum width=\w cm] (leg4) at ([yshift=-1.5em, xshift=0.0em]leg3.south) {};
     \node[right] (txt4) at ([yshift=0em, xshift=0em]leg4.east) {Loss Function};
     
     \node[draw=black, fill=fillcolor6, line width=\figlinewidth pt, rectangle, align=center, anchor=south, minimum height=\w cm, minimum width=\w cm] (leg5) at ([yshift=-1.5em, xshift=0.0em]leg4.south) {};
     \node[right] (txt5) at ([yshift=0em, xshift=0em]leg5.east) {$\bm \downarrow$};
     
     \node[draw=black, fill=fillcolor7, line width=\figlinewidth pt, rectangle, align=center, anchor=south, minimum height=\w cm, minimum width=\w cm] (leg6) at ([yshift=-1.5em, xshift=0.0em]leg5.south) {};
     \node[right] (txt6) at ([yshift=0em, xshift=0em]leg6.east) {$\bm \uparrow$};
     
  \end{tikzpicture}
  }
  \caption{Training process of \methodName\ contains two phases. \textbf{Phase 1:} The closed-form solution for the encoder is calculated using the features generated by the feature extractor while its parameters are frozen. \textbf{Phase 2:} The feature extractor is trained using the loss provided by the calculated encoder parameters from Phase 1.}
  \label{fig:model-overview}
\end{figure*}
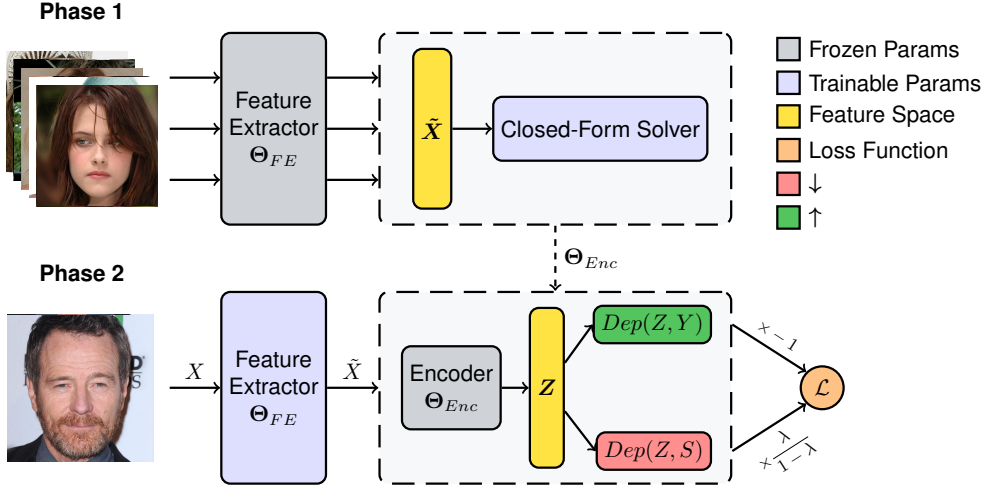